\documentclass[twoside,11pt]{article}
\usepackage{pdf14}
\usepackage{jmlr2e}

\usepackage{soul}
\newcommand{\stkout}[1]{\ifmmode\text{\sout{\ensuremath{#1}}}\else\sout{#1}\fi}

\usepackage{mathrsfs} 
\usepackage{rotating}
\usepackage{times}
\usepackage{hyperref} 
\usepackage{placeins}
\usepackage[utf8]{inputenc} 
\usepackage[T1]{fontenc}    
\usepackage{hyperref}       
\usepackage{enumitem}   
\usepackage{booktabs}       
\usepackage{amsfonts}
\usepackage{nicefrac}       
\usepackage{microtype}      
\usepackage{times}
\usepackage[normalem]{ulem} 
\usepackage{adjustbox}
\usepackage{xcolor}
\usepackage{tabulary}
\usepackage{booktabs}
\usepackage[cal=zapfc,bb=fourier]{mathalfa}

\usepackage{tikz,tikz-cd}
\usetikzlibrary{matrix, calc, arrows}
\usepackage{tabulary,booktabs}
\usetikzlibrary{tikzmark}
\usepackage{stackengine}
\usepackage{thm-restate}
\usepackage{amsmath}
\newtheorem{pred}{Prediction}

\usepackage{algpseudocode,algorithm,algorithmicx}

\makeatletter
\newcommand{\algmargin}{\the\ALG@thistlm}
\makeatother
\newlength{\whilewidth}
\algdef{SE}[parWHILE]{parWhile}{EndparWhile}[1]
  {\parbox[t]{\dimexpr\linewidth-\algmargin}{%
     \hangindent\whilewidth\strut\algorithmicwhile\ #1\ \algorithmicdo\strut}}{\algorithmicend\ \algorithmicwhile}%
\algnewcommand{\parState}[1]{\State%
  \parbox[t]{\dimexpr\linewidth-\algmargin}{\strut #1\strut}}

\DeclareMathOperator*{\arginf}{arg\,inf}

\usepackage{cancel}

\newcommand{\Id}{\textnormal{Id}}
\newcommand{\iid}{\overset{\textnormal{i.i.d}}{\sim}}

\newcommand{\Hess}{\textnormal{H}}
\newcommand{\Diff}{\mathcal{J}}

\newcommand{\Lip}{\textnormal{Lip}}

\ShortHeadings{Risk Bounds for Unsupervised Cross-Domain Mapping}{Tomer Galanti, Sagie Benaim and Lior Wolf}
\firstpageno{1}

\begin{document}

\title{Risk Bounds for Unsupervised\\ Cross-Domain Mapping with IPMs}


\author{}

\author{\name  Tomer Galanti \email tomerga2@post.tau.ac.il\\
		\addr School of Computer Science\\
        Tel Aviv University \\
        Ramat Aviv, Tel Aviv 69978, Israel\\
       \AND
       \name Sagie Benaim \email sagieb@mail.tau.ac.il\\
		\addr School of Computer Science\\
        Tel Aviv University \\
        Ramat Aviv, Tel Aviv 69978, Israel\\
       \AND
       \name Lior Wolf \email wolf@cs.tau.ac.il \\
       \addr Facebook AI Research and\\
       \addr School of Computer Science\\
       Tel Aviv University\\
       Ramat Aviv, Tel Aviv 69978, Israel}

\editor{}

\maketitle

\begin{abstract}
The recent empirical success of unsupervised cross-domain mapping algorithms, between two domains that share common characteristics, is not well-supported by theoretical justifications. This lacuna is especially troubling, given the clear ambiguity in such mappings. 

We work with adversarial training methods {\color{black}based on IPMs} and derive a novel {\color{black}risk} bound, which {\color{black}upper bounds} the risk between the learned mapping $h$ and the target mapping $y$, by a sum of three terms: (i) the risk between $h$ and the most distant alternative mapping that was learned by the same cross-domain mapping algorithm, (ii) the minimal {\color{black}discrepancy} between the target domain and the domain obtained by applying a hypothesis $h^*$ on the samples of the source domain, where $h^*$ is a hypothesis selectable by the same algorithm. The bound is directly related to Occam's razor and encourages the selection of the minimal architecture that supports a small {\color{black}mapping discrepancy and (iii) an approximation error term that decreases as the complexity of the class of discriminators increases and is empirically shown to be small.} 

The bound leads to multiple algorithmic consequences, including a method for hyperparameters selection and for an early stopping in cross-domain mapping GANs. We also demonstrate a novel capability for unsupervised learning of estimating confidence in the mapping of every specific sample. 
\end{abstract}


\section{Introduction}

The recent literature contains many examples of unsupervised learning that are beyond the classical work on clustering and density estimation, most of which revolve around generative models that are trained to capture a certain distribution $D$. In many cases, the generation is unconditioned, and the learned hypothesis takes the form of $g(z)$ for a random vector $z$. It is obtained based on a training set containing i.i.d samples from the target domain. 

{\color{black}A large portion of the recent} literature on this problem employs adversarial training, and specifically a variant of Generative Adversarial Networks (GANs), which were introduced by~\cite{NIPS2014_5423}. GAN-based schemes typically employ two functions that are learned jointly: a generator $g$ and a discriminator $d$. The discriminator is optimized to distinguish between ``real'' training samples from a distribution $D$ and ``fake'' samples that are generated as $g(z)$, where $z$ is distributed according to a predefined latent distribution $D_z$ (typically, a low-dimensional normal or uniform distribution). The generator is optimized to generate adversarial samples, i.e., samples $g(z)$, such that $d$ would classify as real. 

These unconditioned GANs are explored theoretically~\citep{pmlr-v70-arora17a}, and since intuitive non-adversarial (interpolation-based) techniques exist~\citep{bojanowski2017optimizing}, their success is also not surprising.

Much less understood is the ability to learn, in a completely unsupervised manner, in the conditioned case, where the learned function $h$ maps a sample from a source domain $A$ to the analogous sample in a target domain $B$. In this case, we have two distributions $D_A,D_B$ and one aims at mapping a sample $a\sim D_A$ to an analogous sample $h(a)\sim D_B$. This computational problem is known as ``Unsupervised Cross-Domain Mapping'' or ``Image to Image Translation'' when considering visual domains. There are a few issues with this computational problem that cause concern. First, it is unclear what analogous means, let alone to capture it in a formula. Second, as detailed in Sec.~\ref{sec:problemformulation}, the mapping problem is inherently ambiguous. 

Despite these theoretical challenges, the field of unsupervised cross-domain mapping, in which a sample from domain $A$ is translated to a sample in domain $B$, is enjoying a great deal of empirical success, e.g,~\citep{xia2016dual,pmlr-v70-kim17a,CycleGAN2017,dualgan,distgan,liu2017unsupervised}. We attribute this success to what we term the ``Simplicity Hypothesis'', which means that these solutions learn the minimal complexity mapping, such that the discrepancy between the fitted distribution and the target distribution is small. As we show empirically, choosing the minimal complexity mapping eliminates the ambiguity of the problem.

In addition to the empirical validation, we present an {\color{black}upper bound on the risk} that supports the simplicity hypothesis. Bounding the error obtained with unsupervised methods is subject to an inherent challenge: without the ability to directly evaluate the risk on the training set, it is not clear on which grounds to build the bound. Specifically, typical generalization bounds of the form of training risk plus a regularization term cannot be used.

The bound we construct has a different form. As one component, it has the success of the fitting process. This is captured by the {\color{black}mapping discrepancy, measured by the IPM~\citep{muller} between the target distribution $D_B$ and the distribution of generated samples $h \circ D_A$ (i.e., the distribution of $h(a)$ for $a \sim D_A$),} and is typically directly minimized by the learner. Another component is the maximal risk within the hypothesis class to any other hypothesis that also provides a good fit. This term is linked to minimal complexity, since it is expected to be small in minimal hypothesis classes, and it can be estimated empirically for any hypothesis class. 


In addition to explaining the plausibility of unsupervised cross-domain mapping despite the inherent ambiguities, our analysis also directly leads to a set of new unsupervised cross-domain mapping algorithms. By training pairs of networks that are distant from each other and both minimize the {\color{black}mapping discrepancy}, we are able to obtain a measure of confidence on the mapping's outcome. This is surprising for two main reasons: first, in unsupervised settings, confidence estimation is almost unheard of, since it typically requires a second set of supervised samples. Second, confidence is hard to calibrate for multidimensional outputs. The confidence estimation is used for deciding when to stop training (Alg.~\ref{algo:when}) and can be applied for hyperparameters selection (Alg.~\ref{algo:rtrvl}).


\section{Contributions}

The work described here is part of the line of work on the role of minimal complexity in unsupervised learning that we have been following in conference publications~\citep{galanti2018the,benaim2017maximally}. Our contributions in this line of work are as follows.

\begin{enumerate}[leftmargin=*]
    \item Thm.~\ref{thm:boundGEN} provides a rigorous statement of the {\color{black}risk bound} for unsupervised cross-domain mapping with IPMs, which is the basis of this work. This bound sums two terms: (a) The maximal risk within the hypothesis class to any other hypothesis that also provides a good fit. (b) The error of fitting between the two domains. This is captured by the IPM~\citep{muller} that is typically directly minimized by the learner.  
    \item Thm.~\ref{thm:boundGEN} leads to concrete predictions that are verified experimentally in Sec.~\ref{sec:experiments}. In addition, based on this theorem, we introduce Algs.~\ref{algo:when} and~\ref{algo:rtrvl}. The first, serves as a method for deciding when to stop training a generator in unsupervised cross-domain mapping. The second algorithm provides a method for hyperparameters selection for unsupervised cross-domain mapping.  
    \item Our line of work shows that unsupervised cross-domain mapping succeeds when the architecture of the learned generator is of minimal complexity. 
    \item In Sec.~\ref{sec:dualProofExt}, we extend our {\color{black}analysis} for the non-unique case. In this case, there are multiple possible target functions and we wish our algorithm to return a hypothesis that is close to one of them. This extension leads to Alg.~\ref{algo:when2} that extends Alg.~\ref{algo:when}, which is then verified experimentally.
\end{enumerate}

The algorithms presented here, and the empirical results, are {\color{black}extensions of} those in the conference publications, except for Alg.~\ref{algo:when2} that extends Alg.~\ref{algo:when} to the non-unique case, which is new. The contributions in this manuscript over the previous conference publications include: (i) In this paper, we employ Integral Probability Metrics (IPMs), while previous work employed a different measure of discrepancy (a specific type of IPM). (ii) We derive a precise bound for cross-domain mapping (Thm.~\ref{thm:boundGEN}), which was missing in our previous work. While in~\citep{benaim2017maximally}, we provide bounds for unsupervised cross-domain mapping, it is mainly used for motivating the methods and it {\color{black}strongly relies on their ``Occam's razor property'' that does not necessarily hold in practice.} (iii) As mentioned, in Sec.~\ref{sec:dualProofExt}, we extend our analysis for the non-unique case.  



\section{Background}

We briefly review IPMs and WGANs. All notations are listed in Tab.~\ref{tab:summary}.

\subsection{Terminology and Notations}\label{sec:terminology}

We introduce some necessary terminology and notations. 
We denote by $\mathbb{P}$ and $\mathbb{E}$ the probability and expectation operators. We denote by $\Id_{\mathcal{X}}:\mathcal{X} \to \mathcal{X}$ the identity function. For a vector $x \in \mathbb{R}^n$, $\|x\|_2$ denotes the Euclidean norm of $x$ and for a matrix $W \in \mathbb{R}^{m \times n}$, $\|W\|_2 := \max_{x\neq 0} (\|Wx\|_2/\|x\|_2)$ stands for the induced operator norm of $W$. {\color{black} For a given hypothesis class $\mathcal{H}$ and loss function $\ell$, we denote, $\ell_{\mathcal{H}} = \{x \mapsto \ell(h(x),h'(x))\}_{h,h' \in \mathcal{H}}$. For simplicity, when it is clear from the context, instead of writing $\inf_{x \in \mathcal{X}}$ we will write $\inf_x$.}

Let $f:\mathcal{X} \to \mathcal{Y}$ be a function, such that, $\mathcal{X} \subset \mathbb{R}^{m}$ and $\mathcal{Y}\subset \mathbb{R}^{n}$. If $f$ is differentiable, we denote by $\Diff_f(x)$ (or $\nabla f(x)$ when $n=1$) the Jacobian matrix of $f$ in $x$ and if it is twice differentiable, we denote by $\Hess_f(x)$ the Hessian matrix of $f$ in $x$. We denote $f \in C^r$ if $f$ is $r$-times continuously differentiable. 
We define, $\|f\|_{\infty,\mathcal{X}} := \sup\limits_{x \in \mathcal{X}} \|f(x)\|_2$ and $\|f\|_{\Lip} = \sup\limits_{x,y \in \mathcal{X}}(\|f(x)-f(y)\|_2/\|x-y\|_2)$. 
For a twice differentiable function $f$, we denote $\beta(f):= \|\Hess_f\|_{\infty,\mathcal{X}}$. Given a set $E$ and two functions $F:E \to \mathbb{R}$ and $G:E \to \mathbb{R}$, we denote, $F \lesssim G$ if and only if $\exists~C>0~\forall~e \in E: F(e) \leq C\cdot G(e)$. 

\subsection{IPMs and WGANs}\label{sec:ipms}

\paragraph{Integral Probability Metrics} IPMs, first introduced by~\cite{muller}, is a family of pseudometric\footnote{ A pseudometric $d:X^2 \to [0,\infty)$ is a non-negative, symmetric function that satisfies the triangle inequality and $d(x,x)=0$ for all $x \in X$} functions between distributions. Formally, for a given Polish space $\mathcal{S} = (\mathcal{X},\| \cdot \|)$ (i.e., separable and completely metrizable topological space), two distributions $D_1$ and $D_2$ over $\mathcal{X}$ and a class $\mathcal{C}$ of discriminator functions  $d:\mathcal{X} \to \mathbb{R}$, the $\mathcal{C}$-IPM between $D_1$ and $D_2$ is defined as follows: 
\begin{equation}\label{eq:IPM}
\rho_{\mathcal{C}}(D_1,D_2) := \sup\limits_{d \in \mathcal{C}} \Big\{ \mathbb{E}_{x\sim D_1}[d(x)] - \mathbb{E}_{x\sim D_2}[d(x)] \Big\}
\end{equation}
This family of functions includes a wide variety of pseudometric, such as~\citep{DBLP:conf/icml/ArjovskyCB17,DBLP:journals/corr/ZhaoML16,Berthelot2017BEGANBE,Li:2015:GMM:3045118.3045301,NIPS2017_6815,DBLP:conf/nips/MrouehS17,mroueh2018sobolev}. {\color{black}In order to guarantee that $\rho_{\mathcal{C}}$ is non-negative, throughout the paper, we assume that $\mathcal{C}$ is symmetric, i.e., if $d\in \mathcal{C}$, then, $-d \in \mathcal{C}$.}

\paragraph{WGANs optimization} In this work, we give special attention to the WGAN algorithm~\citep{DBLP:conf/icml/ArjovskyCB17} {\color{black}and its extensions~\citep{DBLP:journals/corr/ZhaoML16,Berthelot2017BEGANBE,Li:2015:GMM:3045118.3045301,NIPS2017_6815,DBLP:conf/nips/MrouehS17,mroueh2018sobolev}}. These are variants of GAN, that use {\color{black}the IPM} instead of the original {\color{black} GAN loss}. 
{\color{black}In general, their aim} to find a mapping {\color{black}$g:\mathcal{X}_1 \to \mathcal{X}_2$} (generator) that takes {\color{black}one distribution, $D_1$ over $\mathcal{X}_1$}, and map it into a {\color{black}second distribution $D_2$} by minimizing the distance between {\color{black}$g \circ D_1$} (the distribution of $g(z)$ for {\color{black}$z\sim D_1$}) and {\color{black}$D_2$}.   
Hence, the goal is to select a mapping (generator) $g$ from a class, $\mathcal{H}$, of neural networks of a fixed architecture, that minimizes the {\color{black}{\em mapping discrepancy}}:
{\color{black}
\begin{equation}\label{eq:wganarg}
\begin{aligned}
\rho_{\mathcal{C}}(g \circ D_1,D_2) 
=\sup\limits_{d\in \mathcal{C}} \Big\{ \mathbb{E}_{z\sim D_1}[d(g(z))] - \mathbb{E}_{x\sim D}[d(x)] \Big\}
\end{aligned}
\end{equation}
}
{\color{black}where $g \circ D$ is the distribution of $g(x)$ for $x \sim D$.} For this purpose, {\color{black} these methods make use of finite sets of i.i.d samples $\mathcal{S}_1 = \{z_i\}^{m_1}_{i=1}$ and $\mathcal{S}_2 = \{x_j\}^{m_2}_{j=1}$ from $D_1$ and $D_2$ (resp.). }
The optimization process iteratively minimizes {\color{black}$\Big\{ \frac{1}{m_1}\sum^{m_1}_{i=1}[d(g(z_i))] - \frac{1}{m_2}\sum^{m_2}_{j=1}[d(x_j)] \Big\}$} with respect to $g$ and maximizes it with respect to $d$. In each iteration, the algorithm runs a few gradient based optimization steps for $g$ or $d$. 

{\color{black}In the task of unconditional generation~\citep{DBLP:conf/nips/GoodfellowPMXWOCB14,DBLP:conf/icml/ArjovskyCB17}, the set $\mathcal{X}_1 = \mathcal{X}_z$} is considered a latent space and is typically a convex subset of a Euclidean space, such as $\mathbb{R}^d$, $[-1,1]^d$ or the $d$-dimensional closed unit ball, $\mathbb{B}_d := \{x \in \mathbb{R}^d \mid \|x\|_2 \leq 1\}$.  Additionally, the {\color{black}input distribution $D_1 = D_z$} is typically a normal distribution (for $\mathcal{X}_z = \mathbb{R}^d$) or a uniform distribution (for $\mathcal{X}_z = [-1,1]^d$ or $\mathcal{X}_z = \mathbb{B}_d$). {\color{black}As we see next, we focus on conditional generation~\citep{xia2016dual,pmlr-v70-kim17a,CycleGAN2017,dualgan,distgan,liu2017unsupervised}, where $D_1=D_A$ and $D_2=D_B$ are two distributions over analogous visual domains.}

As a side note, to derive the equation in Eq.~\ref{eq:wganarg}, we used (cf.~\cite{limits}, Thm.~1.9).

\begin{table}[th!]\caption{Summary of Notation}
\begin{center}
\begin{tabular}{r c p{12cm} }
\toprule
$ \ell$ & & The $L_2$ loss function, i.e., $\ell(a,b) = \|a-b\|^2_2$.\\
$\mathbb{P},\mathbb{E}$ &  & The probability and expectation operators\\
$\Id_{\mathcal{X}}$ &  & The identity function\\
$A,B$ &  & Two domains $A=(\mathcal{X}_A,D_A)$ and $B=(\mathcal{X}_B,D_B)$\\
$R_D$, $R_{\mathcal{S}}$ &  & The generalization and empirical risk functions \\
$\mathcal{H},h$ &  & A hypothesis class and a specific hypothesis\\
$\mathcal{C},d$ &  & A class of discriminators and a specific discriminator\\
$\mathcal{T},y$ & & A set of target functions and a specific target function\\
$\rho_{\mathcal{C}}$ &  & The $\mathcal{C}$-IPM\\
$\Omega,\omega$ & & A set of vectors of hyperparamers and a specific vector of hyperparameters\\
$\mathcal{A}_{\omega}$ & & A cross-domain mapping algorithm with hyperparameters $\omega$\\
$\mathcal{P}_{\omega}/\mathcal{P}_{\omega}(\mathcal{S}_A,\mathcal{S}_B)$ & & The set of possible outputs of $\mathcal{A}_{\omega}$  provided with inputs $(\mathcal{S}_A,\mathcal{S}_B)$\\
$\mathcal{H}_{k}$ & & A hypothesis class of functions of complexity $\leq k$\\
$\mathcal{A}_{k}$ & & A cross-domain mapping algorithm of generators from $\mathcal{H}_k$\\
$\mathcal{P}_k/\mathcal{P}_{k}(\mathcal{S}_A,\mathcal{S}_B)$ & & The set of possible outputs of $\mathcal{A}_{k}$ provided with input $(\mathcal{S}_A,\mathcal{S}_B)$\\
$\|x\|_2$, $\|W\|_2$ &  & The Euclidean and induced operator norms\\
$\Diff_f$, $\nabla f$, $\Hess_f$ &  & The Jacobian, gradient and Hessian operators\\
$\|f\|_{\infty,\mathcal{X}}$ & & The infinity norm of $f:\mathcal{X} \to \mathbb{R}^n$ \\
$\|f\|_{\Lip}$ &  & The Lipschitz norm $f:\mathcal{X} \to \mathbb{R}^n$\\
$\beta(f)$ & & The maximal Euclidean norm of the Hessian of $f:\mathcal{X} \to \mathbb{R}$\\
$C^r$ & & The set of $r$-times continuously differentiable functions\\
$F \lesssim G$ & & $\exists~C>0~\forall~e \in E: F(e) \leq C\cdot G(e)$\\
\bottomrule
\end{tabular}
\end{center}
\label{tab:summary}
\end{table}

\section{Problem Setup}
\label{sec:problemformulation}

In this paper, we consider the {\em Unsupervised Cross-Domain Mapping Problem}. In this setting, there are two domains $A = (\mathcal{X}_A,D_A)$ and $B = (\mathcal{X}_B,D_B)$, where $D_A$ and $D_B$ are distributions over the sample spaces $\mathcal{X}_A \subset \mathbb{R}^N$ and $\mathcal{X}_B \subset \mathbb{R}^M$ respectively (formally, we assume that both spaces are equipped with $\sigma$-algebras). In addition, there is a hypothesis class $\mathcal{H}$ of functions $h:\mathcal{X}_A \to \mathbb{R}^M$ and a loss function $\ell:\mathbb{R}^M \times \mathbb{R}^M \to \mathbb{R}$. Our results are shown for the $L_2$-loss $\ell(x_1,x_2) = \|x_1-x_2\|^2_2$.

In this setting, there is an unknown target function $y$ that maps the first domain to the second domain, i.e., $y:\mathcal{X}_A \rightarrow \mathcal{X}_B$ and $D_B = y \circ D_A$. {\color{black}The function $y$ will also be referred as the {\em ``semantic alignment''} between $D_A$ and $D_B$, as opposed to non-semantic alignments $f\neq y$ that map $f \circ D_A = D_B$.} {\color{black}In Sec.~\ref{sec:dualProofExt}}, we extend the framework and the results to include multiple target functions. 

As an example that is often used in the literature, $\mathcal{X}_A$ is a set of images of shoes and $\mathcal{X}_B$ is a set of images of shoe edges, see Fig.~\ref{fig:edges_to_shoes}(a). Here, $D_A$ is a distribution of images of shoes and $D_B$ a distribution of images of shoe edges. The function $y$ takes an image of a shoe and maps it to an image of the edges of the shoe. The assumption that $y\circ D_A = D_B$ simply means that the target function, $y$, takes a sampled image of a shoe $x \sim D_A$ and maps it to a sample $y(x)$ from the distribution of images of edges.

In contrast to the supervised case, where the learning algorithm is provided with a dataset of labeled samples $(x,y(x))$ for $x \sim D$ and $y$ is the target function, in the unsupervised case that we study, the only inputs of the learning algorithm $\mathcal{A}$ are i.i.d samples from the two distributions $D_A$ and $D_B$ independently. 
\begin{equation}
\begin{aligned}
\mathcal{S}_A \iid D^{m_1}_A \textnormal{ and } \mathcal{S}_B \iid D^{m_2}_B
\end{aligned}
\end{equation}
The set $\mathcal{S}_A$ consists of unlabelled instances in $\mathcal{X}_A$ and the set $\mathcal{S}_B$ consists of labels with no sources. We also do not assume that for any $a \in \mathcal{S}_A$ there is a corresponding $b \in \mathcal{S}_B$, such that, $b = y(a)$.

The goal of the learning algorithm $\mathcal{A}$ is to fit a function $h \in \mathcal{H}$ that is closest to $y$,
\begin{equation}
\label{eq:32}
\begin{aligned}
h \in \arginf_{f \in \mathcal{H}} R_{D_A}[f,y] 
\end{aligned}
\end{equation} 
Here, $R_D[f_1,f_2]$ is the generalization risk function between $f_1$ and $f_2$ with respect to a distribution $D$, that is defined in the following manner:
\begin{equation}\label{eq:genRisk}
R_D[f_1,f_2] := \mathbb{E}_{x \sim D} \left[\ell (f_1(x),f_2(x)) \right]
\end{equation}

In supervised learning, the algorithm is provided with the labels of the target function $y$ on the training set $\mathcal{S}_A$ and estimates the generalization risk $R_{D_A}[h,y]$ using the empirical risk $R_{\mathcal{S}_A}[h,y] := \frac{1}{\vert \mathcal{S}_A\vert} \sum_{x \in \mathcal{S}_A} \ell(h(x),y(x))$. In the proposed unsupervised setting, one cannot estimate this risk on the training samples, since the algorithm is not provided with the labeled samples $(x,y(x))$. Instead, the learner must rely on the two independent sets $\mathcal{S}_A$ and $\mathcal{S}_B$. 

With regards to the example above, the learning algorithm is provided with a set of $m_1$ images of shoes and $m_2$ images of shoe edges. The two sets are independent and unmatched. The goal of the learning algorithm is to provide a hypothesis $h$ that approximates $y$. Informally, we want to have $h(a) \approx y(a)$ in expectation over $a \sim D_A$, i.e., $h$ and $y$ map the same image of a shoe to the same image of shoe edges. 

\paragraph{Two modes of failure} {\color{black}Even if the algorithm is provided with an infinitely amount of samples, it can fail in two different ways.} (i) $h$ can fail can fail to produce the output domain, i.e., $h \circ D_A$ will diverge from $D_B$. {\color{black}This is typically a result of limited expressivity.} (ii) Even if $h\circ D_A = D_B$, we can have it map different than $y$, i.e., there would be  a high probability for samples $a \sim D_A$, such that, $\ell(h(a),y(a))$ is large, {\color{black}which is discussed next.}

\subsection{The Unsupervised Alignment Problem}\label{sec:alignment}

\begin{figure}[t]
\centering
\begin{small}
  \begin{tabular}{
  c@{\hspace*{2mm}}
  c@{\hspace*{1mm}}
  c@{\hspace*{1mm}}
  c@{\hspace*{2mm}}
  c@{\hspace*{8mm}}
  c@{\hspace*{2mm}}
  c@{\hspace*{1mm}}
  c@{\hspace*{1mm}}
  c@{\hspace*{2mm}}
  c@{\hspace*{8mm}}
  c@{\hspace*{2mm}}
  c@{\hspace*{1mm}}
  c@{\hspace*{1mm}}
  c@{\hspace*{2mm}}
  c
  }
\includegraphics[width=0.06\linewidth, clip]{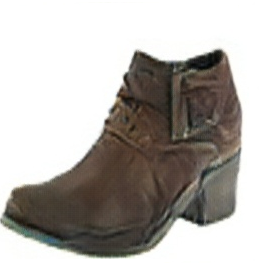}&
\includegraphics[width=0.06\linewidth, clip]{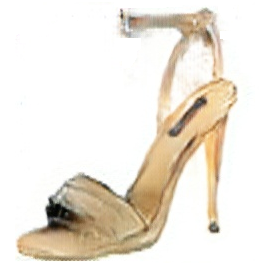}&&
\includegraphics[width=0.06\linewidth, clip]{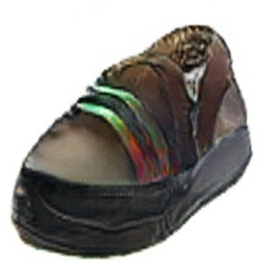}&
\includegraphics[width=0.06\linewidth, clip]{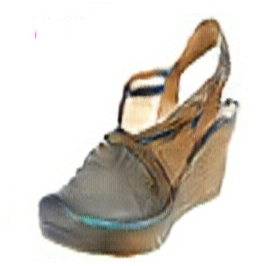}&
\includegraphics[width=0.06\linewidth, clip]{figures_alignment/shoe1.png}&
\includegraphics[width=0.06\linewidth, clip]{figures_alignment/shoe2.png}&&
\includegraphics[width=0.06\linewidth, clip]{figures_alignment/shoe3.png}&
\includegraphics[width=0.06\linewidth, clip]{figures_alignment/shoe4.png}&
\includegraphics[width=0.06\linewidth, clip]{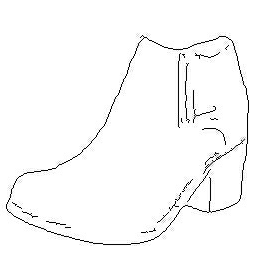}&
\includegraphics[width=0.06\linewidth, clip]{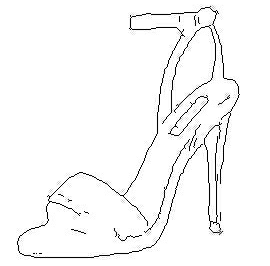}&&
\includegraphics[width=0.06\linewidth, clip]{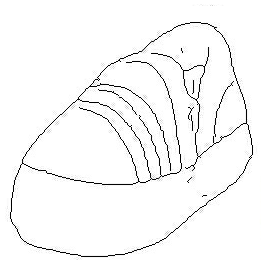}&
\includegraphics[width=0.06\linewidth, clip]{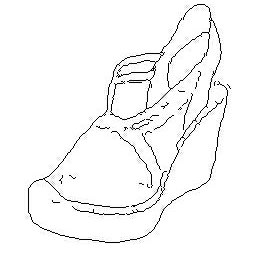}\\
\includegraphics[width=0.06\linewidth, clip]{figures_alignment/edge1.png}&
\includegraphics[width=0.06\linewidth, clip]{figures_alignment/edge2.png}&&
\includegraphics[width=0.06\linewidth, clip]{figures_alignment/edge3.png}&
\includegraphics[width=0.06\linewidth, clip]{figures_alignment/edge4.png}&
\includegraphics[width=0.06\linewidth, clip]{figures_alignment/edge3.png}&
\includegraphics[width=0.06\linewidth, clip]{figures_alignment/edge4.png}&&
\includegraphics[width=0.06\linewidth, clip]{figures_alignment/edge1.png}&
\includegraphics[width=0.06\linewidth, clip]{figures_alignment/edge2.png}&
\includegraphics[width=0.06\linewidth, clip]{figures_alignment/edge3.png}&
\includegraphics[width=0.06\linewidth, clip]{figures_alignment/edge4.png}&&
\includegraphics[width=0.06\linewidth, clip]{figures_alignment/edge1.png}&
\includegraphics[width=0.06\linewidth, clip]{figures_alignment/edge2.png}\\
& & (a) & & & & & (b) & & & & & (c) & &\\
\end{tabular}
\end{small}

\caption{{\bf The alignment problem.} Domain $A$ consists of shoes and domain $B$ consists of edges of shoes. (a) The correct alignment $y$ between the two domains. (b) A wrong alignment $\hat{h}$ between the two domains. The algorithm is provided with independent samples from domain $A$ and from domain $B$. It is not obvious what makes the algorithm return the mapping (a) instead of any other mapping between the two domains. (c) A permutation function $\Pi$ that gives (b) when applied on (a), i.e., $\hat{h} = \Pi \circ y$.}
  \label{fig:edges_to_shoes}
\end{figure}

We next address that the proposed unsupervised learning setting suffers from what we term {\em ``The Alignment Problem''}. The problem arises from the fact that when observing samples {\color{black}$\mathcal{S}_A$ and $\mathcal{S}_B$} only from the marginal distributions $D_A$ and $D_B$, one cannot uniquely link the samples in the source domain to those of the target domain, see Fig.~\ref{fig:edges_to_shoes}(b). 

As a simple example, let $D_A$ and $D_B$ be two discrete distributions, such that, there are two points $a,a' \in \mathcal{X}_A$ that satisfy $\mathbb{P}_{x \sim D_A}[x = a]=\mathbb{P}_{x \sim D_A}[x = a']$. Assuming that the mapping $y$ is one-to-one, then $y(a)$ and $y(a')$ have the same likelihood in the density function $\mathbb{P}_{x \sim D_B}[x = \cdot ]$. Therefore, a-priori it is unclear if the target mapping takes $a$ and maps it to $y(a)$ or to $y(a')$.

More generally, given the target function $y$ between the two domains, in many cases it is possible to define many alternative mappings of the form $\hat{h} = \Pi \circ y$, where $\Pi$ is a mapping that satisfies $\Pi \circ D_B = D_B$. For such functions, we have, $\hat{h} \circ D_A = \Pi \circ y \circ D_A = \Pi \circ D_B = D_B$, and therefore, they satisfy the same assumptions we had regarding the target function $y$.

Thus, a-priori it is unclear why a cross-domain mapping algorithm that only observes samples from $D_A$ and $D_B$ will recover the target mapping $y$ instead of any arbitrary mapping $\hat{h} \neq y$, such that, $\hat{h} \circ D_A = D_B$. In Fig.~\ref{fig:edges_to_shoes}(b), the mapping can be represented as $\hat{h} = \Pi \circ y$, where $\Pi$ is the mapping illustrated in Fig.~\ref{fig:edges_to_shoes}(c).


\subsection{Circularity Constraints do not Eliminate All of the Inherent Ambiguity}

In the field of unsupervised cross-domain mapping, most contributions learn the mapping $h$ between the two domains $A$ and $B$ by employing two constraints. The first, $h$ is restricted to minimize {\color{black}a GAN loss}. In this work, in order to support a more straightforward analysis, we employ {\color{black}IPMs} and $h$ minimizes ${\color{black}\rho_{\mathcal{C}}(h\circ \mathcal{S}_A,\mathcal{S}_B)}$ (Eq.~\ref{eq:IPM}). {\color{black}In~\citep{10.5555/3326943.3327008}, it has been shown that many of the GAN methods in the literature perform similarly.}

A large portion of the cross-domain mapping algorithms also employ what is called the circularity constraint~\citep{xia2016dual,pmlr-v70-kim17a,CycleGAN2017,dualgan}. Circularity requires learning a second mapping $h'$ that maps between $B$ and $A$ (the opposite direction of $h$) and serves as an inverse function to $h$. Similarly to $h$, $h'$ is trained to minimize a {\color{black}GAN loss} but in the other direction, e.g., ${\color{black}\rho_{\mathcal{C}}}(h'\circ \mathcal{S}_B,\mathcal{S}_A)$. The circularity terms, which are minimized by $h$ and $h'$ take the form $R_{\mathcal{S}_A}[h'\circ h,\Id_{\mathcal{X}_A}]$ and $R_{\mathcal{S}_B}[h\circ h',\Id_{\mathcal{X}_B}]$, where $\Id_{\mathcal{X}}:\mathcal{X} \to \mathcal{X}$ is the identity function, i.e., $\forall x \in \mathcal{X}:\Id_{\mathcal{X}}(x)=x$. 
In other words, for a sample {\color{black}$a\in \mathcal{S}_A$}, we expect to have, $h'(h(a))\approx a$ and for a random sample {\color{black} $b \in \mathcal{S}_B$}, we expect to have, $h(h'(b))\approx b$. 

Therefore, the complete minimization objective of both $h$ and $h'$ is as follows:
\begin{subequations}
\label{eq:circ}
\begin{align}
\inf\limits_{h,h'\in \mathcal{H}} &{\color{black}\rho_{\mathcal{C}}(h\circ \mathcal{S}_A,\mathcal{S}_B) + \rho_{\mathcal{C}}(h'\circ \mathcal{S}_B,\mathcal{S}_A) }\label{eq:circa}\\
&{\color{black}+ R_{\mathcal{S}_A}[h'\circ h,\Id_{\mathcal{X}_A}] + R_{\mathcal{S}_B}[h\circ h',\Id_{\mathcal{X}_B}]}
\label{eq:circb}
\end{align}
\end{subequations} 
The terms in Eq.~\ref{eq:circa} ensure that the samples generated by mapping domain $A$ to domain $B$ follow the distribution of samples in domain $B$ and vice versa. The terms in Eq.~\ref{eq:circb} ensure that mapping a sample from one domain to the second and back, results in the original sample. Note that the first two terms match distributions (via the {\color{black}IPM} scores) and the last two match individual samples (via the loss $\ell$ in the risk).

The circularity terms are shown empirically to improve the obtained results. However, these terms do not eliminate all of the inherent ambiguity, as shown in the following observation. {\color{black} For instance, consider the favorable case where the algorithm has full access to $D_A$ and $D_B$, i.e., $\mathcal{S}_A = D_A$ and $\mathcal{S}_B = D_B$.} Let $\Pi$ be an invertible permutation of $D_B$, i.e., $\Pi:\mathcal{X}_B \to \mathcal{X}_B$ is an invertible mapping and $\Pi \circ D_B = D_B$. Then, the pair $h = \Pi \circ y$ and $h' = y^{-1} \circ \Pi^{-1}$ achieves:
\begin{equation}
{\color{black}\rho_{\mathcal{C}}}(h\circ D_A,D_B) + {\color{black}\rho_{\mathcal{C}}}(h'\circ D_B,D_A) + R_{D_A}[h'\circ h,\Id_{\mathcal{X}_A}] + R_{D_B}[h\circ h',\Id_{\mathcal{X}_B}] = 0
\end{equation}
Informally, if $\Pi$ is an invertible permutation of the samples in domain $B$ ({\em not} a permutation of the vector elements of the representation of samples in $B$), then, if $y$ is the target function and $y^{-1}$ is its inverse function, the pair of functions $h = \Pi \circ y$ and $h' = y^{-1} \circ \Pi^{-1}$ achieves zero losses. Therefore, even though the function $h = \Pi \circ y$ might correspond to an incorrect alignment between the two domains $A$ and $B$ (i.e., the function $h$ is very different from $y$), the pair $h$ and $h'$ can still achieve a zero value on each of the losses proposed by~\citep{xia2016dual,pmlr-v70-kim17a,CycleGAN2017,dualgan}. 

Since both low {\color{black}discrepancy} and circularity cannot, separately or jointly, eliminate the ambiguity of the mapping problem, a complete explanation of the success of unsupervised cross-domain mapping must consider the hypothesis classes $\mathcal H$ and $\mathcal C$. This is what we intend to do in Sec.~\ref{sec:dualProof}.

\subsection{Cross-Domain Mapping Algorithms}\label{sec:algorithms}

A central goal in this work, is the derivation of {\color{black}risk} bounds that can be used to compare different cross-domain mapping algorithms. The set of cross-domain mapping algorithms $\{\mathcal{A}_{\omega}\}_{\omega \in \Omega}$ that are compared, are indexed by a vector of hyperparameters $\omega \in \Omega$. The vector of hyperparameters $\omega$ can include the architecture of the hypothesis class from which $\mathcal{A}_{\omega}$ selects candidates, the learning rate, batch size, etc'. To compare the performance of the algorithms, an upper bound on the term $R_{D_A}[h,y]$ is provided. Fortunately, this bound can be estimated without the need for supervised data, i.e., without paired matches $(x,y(x))$. Here, $h$ is the selected hypothesis by a cross-domain mapping algorithm $\mathcal{A}_{\omega}$ provided with access to the distributions $D_A$ and $D_B$. 

The outcome of every deep learning algorithm often depends on the random initialization of its parameters and the order in which the samples are presented. Such non-deterministic algorithm $\mathcal{A}_{\omega}$ can be seen as a mapping from the training data {\color{black}$(\mathcal{S}_A,\mathcal{S}_B)$} to a subset of the hypothesis space $\mathcal{H}$. This subset, which is denoted as {\color{black}$\mathcal{P}_{\omega}(\mathcal{S}_A,\mathcal{S}_B)$}, contains all the hypotheses that the algorithm may return for the given training data. 
Typically, the set {\color{black}$\mathcal{P}_{\omega}(\mathcal{S}_A,\mathcal{S}_B)$} is much sparser than the original hypothesis class.  Since the algorithm is not assumed to be deterministic, to measure the performance of a cross-domain mapping algorithm $\mathcal{A}_{\omega}$, an upper bound on $R_{D_A}[h,y]$ is derived for any {\color{black}$ h\in \mathcal{P}_{\omega}(\mathcal{S}_A,\mathcal{S}_B)$}. {\color{black}For simplicity, sometimes we will simply write $\mathcal{P}_{\omega}$ as a reference to $\mathcal{P}_{\omega}(\mathcal{S}_A,\mathcal{S}_B)$.}

In this paper, special attention is given to IPM minimization algorithms applied for unsupervised cross-domain mapping. The algorithm, $\mathcal{A}_{\omega}$, given access to two {\color{black}datasets $\mathcal{S}_A$ and $\mathcal{S}_B$, a hypothesis class $\mathcal{H}_{\omega}$ and a class of discriminators $\mathcal{C}$, returns a hypothesis $h \in \mathcal{H}_{\omega}$ (see Eq.~\ref{eq:IPM}), that minimizes $\rho_{\mathcal{C}}(h \circ \mathcal{S}_A,\mathcal{S}_B)$.} 

The following are concrete examples of the proposed framework. We specify, $\omega$, $\mathcal{H}$, $\mathcal{H}_{\omega}$ and {\color{black}$\mathcal{P}_{\omega}(\mathcal{S}_A,\mathcal{S}_B)$} for different settings: 

\begin{enumerate}
    \item \textbf{The hyperparameters are the learning rate and batch size:} in this case each $\omega = (\mu,s,T)$ includes a different learning rate $\mu>0$, batch size $s \in \mathbb{N}$ and number of iterations $T \in \mathbb{N}$. The hypothesis class $\mathcal{H}_{\omega}$ from which $\mathcal{A}_{\omega}$ selects candidates is $\mathcal{H}$ itself. The set {\color{black}$\mathcal{P}_{\omega}(\mathcal{S}_A,\mathcal{S}_B)$} consists of {\color{black}the minimizers $h\in \mathcal{H}$ of $\rho_{\mathcal{C}}(h \circ \mathcal{S}_A,\mathcal{S}_B)$ trained using the procedure in Sec.~\ref{sec:ipms} with a learning rate $\mu$, a batch size $s$ and the number of iterations $T$.} 
    \item \textbf{The hyperparameter is the number of layers:} in this case $\mathcal{H}$ is a class of neural networks of varying number of layers, each of size $\in [r_1,r_2]$, for some predefined $r_1,r_2 \in \mathbb{N}$. 
    The hyperparameter is the maximal number of layers $\omega = k \in \mathbb{N}$ of the trained neural network $h \in \mathcal{H}$. The hypothesis class $\mathcal{H}_{k}$ is the set of neural networks in $\mathcal{H}$ that have a depth $\leq k$. {\color{black}The algorithm $\mathcal{A}_k$ returns a hypothesis $h\in \mathcal{H}_k$ that minimizes $\rho_{\mathcal{C}}(h \circ \mathcal{S}_A,\mathcal{S}_B)$. More formally, $\mathcal{P}_k(\mathcal{S}_A,\mathcal{S}_B) = \{h \in \mathcal{H}_k \mid \rho_{\mathcal{C}}(h \circ \mathcal{S}_A,\mathcal{S}_B) \leq c \inf_{h^* \in \mathcal{H}_k} \rho_{\mathcal{C}}(h^* \circ \mathcal{S}_A,\mathcal{S}_B)\}$ for some predefined multiplicative tolerance parameter $c\geq 1$ of our choice.} 
    \item \textbf{The hyperparameters are the weights of the first layers:} in this case, $\mathcal{H} = \{h_{\theta,\omega} = g_{\theta} \circ f_{\omega} \vert \theta \in \Theta, \omega \in \Omega\}$ is a set of neural networks, each parameterized by two sets of parameters $\theta \in \Theta$ and $\omega \in \Omega$. We can think of $f_{\omega}$ as the first $l_1$ layers of $h_{\theta,\omega}$ and $g_{\theta}$ as the last $l_2$ layers of it. For instance, $f_{\omega}$ can be an encoder and $g_{\theta}$ a decoder. Here, $\mathcal{H}_{\omega} = \{h_{\theta,\omega} \vert \theta \in \Theta\}$ is the set of neural networks $h_{\theta,\omega} \in \mathcal{H}$ with fixed $\omega$. In this setting, {\color{black}$\mathcal{P}_{\omega}(\mathcal{S}_A,\mathcal{S}_B)$ consists of the set of the possible minimizers $h \in \mathcal{H}_{\omega}$ of $\rho_{\mathcal{C}}(h \circ \mathcal{S}_A,\mathcal{S}_B)$. More formally, $\mathcal{P}_{\omega}(\mathcal{S}_A,\mathcal{S}_B) = \{h \in \mathcal{H}_{\omega} \mid \rho_{\mathcal{C}}(h \circ \mathcal{S}_A,\mathcal{S}_B) \leq c \inf_{h^* \in \mathcal{H}_{\omega}} \rho_{\mathcal{C}}(h^* \circ \mathcal{S}_A,\mathcal{S}_B)\}$ for some predefined multiplicative tolerance parameter $c\geq 1$ of our choice.}  
\end{enumerate}

{\color{black}
In general, our bounds hold for any set of classes $\{\mathcal{P}_{\omega}(\mathcal{S}_A,\mathcal{S}_B)\}_{\omega \in \Omega}$ (not even minimizers of the mapping discrepancy). However, to obtain the simplified analysis in Sec.~\ref{sec:genB}, we consider sets $\mathcal{P}_{\omega}(\mathcal{S}_A,\mathcal{S}_B)$, such that, the hypotheses $h\in \mathcal{P}_{\omega}(\mathcal{S}_A,\mathcal{S}_B)$ achieve a fairly similar degree of success at minimizing $\rho_{\mathcal{C}}(h \circ \mathcal{S}_A,\mathcal{S}_B)$, i.e., there is a multiplicative tolerance constant $c\geq 1$, such that, for all $h \in \mathcal{P}_{\omega}(\mathcal{S}_A,\mathcal{S}_B)$, we have: $\rho_{\mathcal{C}}(h \circ \mathcal{S}_A,\mathcal{S}_B) \leq c \inf_{h^* \in \mathcal{P}_{\omega}}\rho_{\mathcal{C}}(h^* \circ \mathcal{S}_A,\mathcal{S}_B)$.
We note that this assumption is not necessary to our analysis and one can obtain a more general version of Thm.~\ref{thm:boundGEN} without it (see Lem.~\ref{lem:boundGENE} in Appendix~\ref{sec:proofs}). Specifically, this condition is met for examples (2) and (3) above. We note that $\mathcal{P}_k(\mathcal{S}_A,\mathcal{S}_B)$ and $\mathcal{P}_{\omega}(\mathcal{S}_A,\mathcal{S}_B)$ are always non-empty. In Sec.~\ref{sec:est}, we show how the constraint $\rho_{\mathcal{C}}(h \circ \mathcal{S}_A,\mathcal{S}_B) \leq c \inf_{h^* \in \mathcal{H}_k} \rho_{\mathcal{C}}(h^* \circ \mathcal{S}_A,\mathcal{S}_B)$ can be carried out in practice.



}

\section{Risk Bounds for Unsupervised Cross-Domain Mapping}
\label{sec:dualProof}

In this section, we discuss sufficient conditions for overcoming the alignment problem. For simplicity, we first focus on the unique case, i.e., there is a unique target function $y$. The results are then extended, in Sec.~\ref{sec:dualProofExt}, to the non-unique case, where there are multiple target functions.

\subsection{Risk Bounds}\label{sec:genB}

We derive {\color{black}risk} bounds for unsupervised cross-domain mapping, comparing alternative {\color{black} hyperparameters $\omega \in \Omega$}. 
As discussed in Sec.~\ref{sec:problemformulation}, our goal is to select $\omega$ that provides the best performing algorithm $\mathcal{A}_{\omega}$ in terms of minimizing $R_{D_A}[h_1,y]$ for an output $h_1$ of $\mathcal{A}_{\omega}$ provided with access to $D_A$ and $D_B$. For this purpose, Thm.~\ref{thm:boundGEN} will provide us with an upper bound (for every $\omega \in \Omega$) on the generalization risk $R_{D_A}[h_1,y]$, for an arbitrary hypothesis $h_1$ selected by the algorithm $\mathcal{A}_{\omega}$, i.e., {\color{black}$h_1 \in \mathcal{P}_{\omega}(\mathcal{S}_A,\mathcal{S}_B)$ }. The terms of the bound can be estimated in an unsupervised manner, see Sec.~\ref{sec:conseq} for the derived algorithms. 

{\color{black}
Our risk bounds take into account the transition between the population distribution and an empirical set of samples from it. This analysis is based on the following data-dependent measure of the complexity of a class of functions.

\begin{definition}[Rademacher Complexity] Let $\mathcal{H}$ be a set of real-valued functions $f:\mathcal{X} \to \mathbb{R}$ defined over a set $\mathcal{X}$. Given a fixed sample $S \in \mathcal{X}^m$, the empirical Rademacher complexity of $\mathcal{H}$ is defined as follows: 
\begin{equation}
\hat{\mathscr{R}}_{S}(\mathcal{H}) := \frac{2}{m} \mathbb{E}_{\sigma} \left[ \sup_{h \in \mathcal{H}} \Big\vert \sum^{m}_{i=1} \sigma_i h(x_i) \Big\vert \right]
\end{equation}
The expectation is taken over $\sigma = (\sigma_1,\dots,\sigma_m)$, where, $\sigma_i \in \{\pm 1\}$ are i.i.d and uniformly distributed samples. 
\end{definition}
The Rademacher complexity measures the ability of a class of functions to fit noise. The empirical Rademacher complexity has the added advantage that it is data-dependent and can be measured from finite samples. It can lead to tighter bounds than those based on other measures of complexity such as the VC-dimension~\citep{koltch}. 
}


The following theorem introduces a {\color{black}risk} bound for unsupervised cross-domain mapping.


{\color{black}
\begin{restatable}[Cross-Domain Mapping with IPMs]{thrm}{boundGEN}\label{thm:boundGEN} Assume that $\mathcal{X}_A \subset \mathbb{R}^N$ and $\mathcal{X}_B \subset \mathbb{R}^M$ are convex and bounded sets. Let $\mathcal{H}$ be the hypothesis class and $\mathcal{C}$ the class of discriminators. Assume that $\mathcal{C} \subset C^2$ and $\sup_{d \in \mathcal{C}}\|d\|_{\infty,\mathcal{X}_A \cup \mathcal{X}_B} < \infty$. 
Then, for any $\delta \in (0,1)$ and $c\geq 1$,  with probability at least $1-\delta$ over the selection of $\mathcal{S}_A \sim D^{m_1}_A$ and $\mathcal{S}_B \sim D^{m_2}_B$, for every $\omega \in \Omega$ and  $h_1 \in \mathcal{P}_{\omega} := \mathcal{P}_{\omega}(\mathcal{S}_A,\mathcal{S}_B)$, we have:
\begin{equation}\label{eq:main}
\begin{aligned}
 R_{D_A}[h_1,y] \lesssim& \sup_{h_2 \in \mathcal{P}_{\omega}} R_{\mathcal{S}_A}[h_1,h_2] + c \inf_{h \in \mathcal{P}_{\omega}} \rho_{\mathcal{C}}(h \circ \mathcal{S}_A,\mathcal{S}_B) + \inf\limits_{h \in \mathcal{P}_{\omega}} \inf_{\substack{d \in \mathcal{C} \\ \beta(d)\leq 1}} \mathcal{K}(h,d;y)\\
&+ \hat{\mathscr{R}}_{\mathcal{S}_A}(\ell_\mathcal{H}) + \hat{\mathscr{R}}_{\mathcal{S}_A}(\mathcal{C}\circ \mathcal{H}) + \hat{\mathscr{R}}_{\mathcal{S}_B}(\mathcal{C}) + \sqrt{\frac{\log(1/\delta)}{\min(m_1,m_2)}}
\end{aligned}
\end{equation}
where, $\mathcal{K}(h,d;y) := \mathbb{E}_{x \sim D_A}\left[\|\nabla_{y(x)} d(y(x)) - (h(x)-y(x))\|_2 \right]$.
\end{restatable}
}
The proof of this theorem can be found in Sec.~\ref{sec:proofs} of the Appendix.

\subsection{Analyzing the Bound}\label{sec:analyzing}

Thm.~\ref{thm:boundGEN} provides an upper bound on the generalization risk, $R_{D_A}[h_1,y]$, of a hypothesis $h_1$ that was selected by an algorithm $\mathcal{A}_{\omega}$, which is the argument that we would like to minimize. 



This bound is decomposed into {\color{black}four parts}. The first term{\color{black}, $\sup_{h_2 \in \mathcal{P}_{\omega}} R_{\mathcal{S}_A}[h_1,h_2]$,} measures the maximal distance between $h_1$ and a second candidate $h_2 \in \mathcal{P}_{\omega}(\mathcal{S}_A,\mathcal{S}_B)$. The second and third terms behave as approximation errors. The second term{\color{black}, $c \cdot \inf_{h \in \mathcal{P}_{\omega}} \rho_{\mathcal{C}}(h \circ \mathcal{S}_A,\mathcal{S}_B)$,} measures the discrepancy between the distributions {\color{black} $h \circ \mathcal{S}_A$} and {\color{black} $\mathcal{S}_B$} for the best fitting hypothesis {$h \in \mathcal{P}_{\omega}(\mathcal{S}_A,\mathcal{S}_B)$}. This is captured by the $\mathcal{C}$-IPM between {\color{black} $h \circ \mathcal{S}_A$} and {\color{black} $\mathcal{S}_B$}. {\color{black} In Sec.~\ref{sec:est}, we show how these terms are being estimated.}

{\color{black} The fourth part (including three Rademacher complexities and the square root) is a result of the transition from empirical to expected quantities. It consists of the empirical Rademacher complexities of the classes $\ell_\mathcal{H}$, $\mathcal{C} \circ \mathcal{H}$ and $\mathcal{C}$ and the term $\sqrt{\frac{\log(1/\delta)}{\min(m_1,m_2)}}$. These terms are standard when considering generalization bounds. When these classes have a finite pseudo-dimension, which is the typical case of neural networks, their corresponding Rademacher complexities are of order $\mathcal{O}\left(\sqrt{\log(m_i)/m_i}\right)$~\citep{Mohri:2012:FML:2371238}. Several publications, e.g.,~\citep{Bartlett:2017:SMB:3295222.3295372,pmlr-v75-golowich18a}, showed that the Rademacher complexity of neural networks is proportional to the spectral norm of the neural networks. For simplicity, we neglect these terms since we focus on the conditions for solving the unsupervised learning task, rather on the generalization capabilities of the algorithm.

The third term, $\mathcal{K}:=\inf_{h,d} \mathcal{K}(h,d;y)$, serves as a mutual approximation error of the classes $\mathcal{P}_{\omega}(\mathcal{S}_A,\mathcal{S}_B)$ and $\nabla \mathcal{C} := \{\nabla d \mid d \in \mathcal{C}\}$. We note that if the zero function $d_0\equiv 0$ is a member of $\mathcal{C}$, then, $\mathcal{K}(h,d_0;y) \leq \mathbb{E}_{x \sim D_A}[\|h(x)-y(x)\|_2]$ and when considering setting (2) in Sec.~\ref{sec:algorithms}, we have:
\begin{equation}\label{eq:benaim}
\begin{aligned}
 R_{D_A}[h_1,y] \lesssim& \sup\limits_{h_2 \in \mathcal{P}_k} R_{\mathcal{S}_A}[h_1,h_2] + \inf\limits_{h \in \mathcal{P}_k} \rho_{\mathcal{C}}(h \circ \mathcal{S}_A,\mathcal{S}_B) + \inf_{h \in \mathcal{P}_k} \mathbb{E}_{x \sim D_A}[\|h(x)-y(x)\|_2]\\
&+ \hat{\mathscr{R}}_{\mathcal{S}_A}(\ell_\mathcal{H}) + \hat{\mathscr{R}}_{\mathcal{S}_A}(\mathcal{C}\circ \mathcal{H}) + \hat{\mathscr{R}}_{\mathcal{S}_B}(\mathcal{C}) + \sqrt{\frac{\log(1/\delta)}{\min(m_1,m_2)}}
\end{aligned}
\end{equation}
{\color{black}which is essentially the bound in~\citep{benaim2017maximally}. The main disadvantage of Eq.~\ref{eq:benaim} follows from the fact that the bound is tight only when $\inf_{h\in \mathcal{P}_k} \mathbb{E}_{x \sim D_A}[\|h(x)-y(x)\|_2]$ is small, i.e., there is a good approximator $h$ of $y$. We note that for a large enough complexity $k$, we expect $\inf_{h\in \mathcal{P}_k} \mathbb{E}_{x \sim D_A}[\|h(x)-y(x)\|_2]$ to be small. However, for larger values of $k$, we also expect $\sup_{h_2 \in \mathcal{P}_k} R_{\mathcal{S}_A}[h_1,h_2]$ to be larger. This is especially crucial as it indicates that the bound in~\citep{benaim2017maximally} is effective only when the term $\inf_{h\in \mathcal{P}_k} \mathbb{E}_{x \sim D_A}[\|h(x)-y(x)\|_2]$ is small for small values of $k$. This property is termed ``Occam's razor property'' by~\citep{benaim2017maximally}. On the other hand, the term $\mathcal{K}$ in Thm.~\ref{thm:boundGEN} decreases when increasing the capacity of $\mathcal{C}$. In particular, for a wide class of discriminators $\mathcal{C}$, we do not have to assume the existence of a particularly good approximator $h\in\mathcal{P}_{k}(\mathcal{S}_A,\mathcal{S}_B)$ of $y$ in order to guarantee that the value of $\mathcal{K}$ is small, in contrast to the analysis in~\citep{benaim2017maximally}. Therefore, our bound does not rely on the assumption that $y$ can be approximated by small complexity networks. In Sec.~\ref{sec:estK} we empirically compare between the two terms.}

Finally, when assuming that $\inf_{h,d}\mathcal{K}(h,d;y)$ is small and neglecting the generalization gap terms, we have:
\begin{equation}\label{eq:simple}
R_{D_A}[h_1,y] \lesssim \sup\limits_{h_2 \in \mathcal{P}_{\omega}} R_{\mathcal{S}_A}[h_1,h_2] + \inf_{h \in \mathcal{P}_{\omega}}\rho_{\mathcal{C}}(h \circ \mathcal{S}_A, \mathcal{S}_B)
\end{equation}
This inequality is a cornerstone in the derivation of the predictions and algorithms for cross-domain mapping presented in Sec.~\ref{sec:conseq}.


}

\section{Consequences of the Bound}\label{sec:conseq}

Thm.~\ref{thm:boundGEN} leads to concrete predictions, {\color{black}when applied for setting (2) in Sec.~\ref{sec:algorithms}.} The predictions are verified in Sec.~\ref{sec:experiments}. The first one states that in contrast to the current common wisdom, one can learn a semantically aligned mapping between two spaces without any matching samples and even without circularity constraints. 

\begin{pred}\label{pred1}
{\color{black}In unsupervised cross-domain mapping, one can obtain the results of a GAN and circularity losses based method (i.e., Eq.~\ref{eq:circ}) using the same architecture, even without the circularity losses.}
\end{pred} 

The strongest clue that helps identify the alignment of the semantic mapping from the other mappings, is the suitable complexity of the network that is learned. A network with a complexity that is too low cannot replicate the target distribution, when taking inputs in the source domain (high discrepancy). A network that has a complexity that is too high, would not learn the minimal complexity mapping, since it could be distracted by other alignment solutions.

We believe that the success of the recent methods results from selecting the architecture used in an appropriate way. For example, DiscoGAN~\citep{pmlr-v70-kim17a} employs either eight or ten layers, depending on the dataset. We make the following prediction:

\begin{pred}\label{pred2}
{\color{black} The term $\inf_{h \in \mathcal{P}_{k}} \rho_{\mathcal{C}}(h \circ \mathcal{S}_A, \mathcal{S}_B)$ decreases as $k$ increases and $\sup_{h_2 \in \mathcal{P}_{k}} R_{\mathcal{S}_A}[h_1,h_2]$ increases as $k$ increases. Therefore, to make both of the terms small, it is preferable to select the minimal complexity $k \in \mathbb{N}$ that provides a hypothesis $h \in \mathcal{H}_k$ that has a small $\rho_{\mathcal{C}}(h \circ \mathcal{S}_A, \mathcal{S}_B)$.}
\end{pred}

This prediction is also surprising, since in supervised learning, extra {\color{black} complexity} is not as detrimental, {\color{black} as long as the training dataset is large enough.} As far as we know, this is the first time that this clear distinction between supervised and unsupervised learning is made\footnote{The Minimum Description Length (MDL for short) literature was developed when people believed that small hypothesis classes are desired for both supervised and unsupervised learning.}.

\subsection{Estimating the Ground Truth Error}
\label{sec:est}

{\color{black}The statement in} Eq.~\ref{eq:simple} provides us with an accessible upper bound for the generalization risk,
{\color{black}
\begin{equation}\label{eq:boundEps0}
R_{D_A}[h_1,y] \lesssim \sup\limits_{h_2 \in \mathcal{P}_k} R_{\mathcal{S}_A}[h_1,h_2] + c \inf_{h \in \mathcal{P}_k} \rho_{\mathcal{C}}(h \circ \mathcal{S}_A, \mathcal{S}_B)
\end{equation}
where $\mathcal{P}_k = \mathcal{P}_k(\mathcal{S}_A,\mathcal{S}_B)$ is chosen to be $\mathcal{P}_k := \{h \in \mathcal{H}_k \mid \rho_{\mathcal{C}}(h \circ \mathcal{S}_A,\mathcal{S}_B) \leq c \rho^*_k\}$, for some predefined tolerance parameter $c\geq 1$ (see Sec.~\ref{sec:algorithms}) and $\rho^*_k := \inf_{h \in \mathcal{H}_k} \rho_{\mathcal{C}}(h \circ \mathcal{S}_A,\mathcal{S}_B)$. We would like to compute an approximation of the RHS of Eq.~\ref{eq:boundEps0} that also upper bounds it. For this purpose, we will show how to compute upper bounds of $\sup_{h_2 \in \mathcal{P}_k} R_{\mathcal{S}_A}[h_1,h_2]$ and $\inf_{h \in \mathcal{P}_k} \rho_{\mathcal{C}}(h \circ \mathcal{S}_A, \mathcal{S}_B)$. 

To upper bound the second term, we note that for any fixed $k \in \mathbb{N}$, by the definition of $\mathcal{P}_k$, we have, 
\begin{equation}
\inf_{h \in \mathcal{P}_k} \rho_{\mathcal{C}}(h \circ \mathcal{S}_A, \mathcal{S}_B) = \rho^*_k,    
\end{equation}
which is a constant. To estimate this term, we train a hypothesis $h \in \mathcal{H}_k$ to minimize $\rho_{\mathcal{C}}(h \circ \mathcal{S}_A,\mathcal{S}_B)$ as discussed in Sec.~\ref{sec:ipms}. This produces an upper bound $\hat{\rho}^*_k := \rho_{\mathcal{C}}(h \circ \mathcal{S}_A,\mathcal{S}_B)$ of the second term $\rho^*_k$. 

Next, we would like to show how to estimate the first term. Namely, for each $h_1$, we would like to solve the following objective:
\begin{equation}\label{eq:primal}
\begin{aligned}
\max\limits_{h_2} R_{\mathcal{S}_A}[h_1,h_2] \textnormal{ s.t: } h_2 \in \mathcal{P}_{k}
\end{aligned}
\end{equation}
or alternatively, 
\begin{equation}\label{eq:primal2}
\begin{aligned}
\max\limits_{h_2 \in \mathcal{H}_k} R_{\mathcal{S}_A}[h_1,h_2] \textnormal{ s.t: } \rho_{\mathcal{C}}(h_2 \circ \mathcal{S}_A,\mathcal{S}_B) \leq c \rho^*_k
\end{aligned}
\end{equation}
Since we do not know the value of the value of $\rho^*_k$ explicitly, instead, we consider the following relaxed version of Eq.~\ref{eq:primal2}:
\begin{equation}\label{eq:relaxedprimal}
\max\limits_{h_2 \in \mathcal{H}_k} R_{\mathcal{S}_A}[h_1,h_2] \textnormal{ s.t: } \rho_{\mathcal{C}}(h_2 \circ \mathcal{S}_A,\mathcal{S}_B) \leq c \hat{\rho}^*_k
\end{equation}
We note that any $h_2 \in \mathcal{H}_k$ that satisfies the condition in Eq.~\ref{eq:primal2} also satisfies the condition in  Eq.~\ref{eq:relaxedprimal}, since $\rho^*_k \leq \hat{\rho}^*_k$. Therefore, the value in Eq.~\ref{eq:relaxedprimal} is an upper bound on the value in Eq.~\ref{eq:primal}. To summarize, if $h^*_2$ is the solution to Eq.~\ref{eq:relaxedprimal}, the expression $R_{\mathcal{S}_A}[h_1,h^*_2] + c \hat{\rho}^*_k$ upper bounds the RHS in Eq.~\ref{eq:boundEps0}.} Finally, in order to train $h_2$, inspired by Lagrange relaxation, we employ the following relaxed version of it:
{\color{black}
\begin{equation}\label{eq:dual}
\min\limits_{h_2\in \mathcal{H}_k} \Big\{ \rho_{\mathcal{C}}(h_2 \circ \mathcal{S}_A,\mathcal{S}_B) - \lambda R_{\mathcal{S}_A}[h_1,h_2] \Big\}
\end{equation}
To minimize the term $\rho_{\mathcal{C}}(h_2 \circ \mathcal{S}_A,\mathcal{S}_B)$ within the objective in Eq.~\ref{eq:dual}, we train $h_2$ against a discriminator as discussed in Sec.~\ref{sec:ipms}. Throughout the optimization process of $h_2$, we can only keep instances of it if $\rho_{\mathcal{C}}(h_2 \circ \mathcal{S}_A,\mathcal{S}_B) \leq c \hat{\rho}^*_k$ is valid. 

It is worth noting that the value of $c$ is a matter of choice. In principle, we could select $c=1$ and the bound would still be valid. However, in practice, it is advantageous to take $c>1$ as it lets us reject fewer candidates $h_2$.} Based on this, we present a stopping criterion in Alg.~\ref{algo:when}. Eq.~\ref{eq:dual} is manifested in Step 6. 


\begin{algorithm}[t]
\caption{Deciding when to stop training $h_1$}
\label{algo:when}
  \begin{algorithmic}[1]
  \Require{$\mathcal{S}_A$ and $\mathcal{S}_B$: unlabeled training sets; $\mathcal{H}$: a hypothesis class; {\color{black}$\mathcal{C}$: a class of discriminators;} {\color{black}$c$: a tolerance scale;} $k$: a complexity value;
  $\lambda$: a trade-off parameter; $T_1$: a fixed number of epochs for $h_1$; $T_2$: a fixed number of epochs for $h_2$.}
  	\State Initialize $h^0_1\in\mathcal{H}_k$ and $h^0_2\in\mathcal{H}_k$ at random.
  	\State {\color{black} Train a hypothesis $h \in \mathcal{H}_k$ to minimize $\rho_{\mathcal{C}}(h \circ \mathcal{S}_A,\mathcal{S}_B)$.} 
  	\State {\color{black} Define $\hat{\rho}^*_k := \rho_{\mathcal{C}}(h \circ \mathcal{S}_A,\mathcal{S}_B)$ (upper bounds $\inf_{h \in \mathcal{H}_k}\rho_{\mathcal{C}}(h \circ \mathcal{S}_A,\mathcal{S}_B)$).} 
    \For {$t = 1,\dots,T_1$}
      \parState{%
        Train $h^{t-1}_1 \in \mathcal{H}_k$ for one epoch to minimize {\color{black}\begin{small}$\rho_{\mathcal{C}}(h^t_1\circ \mathcal{S}_A,\mathcal{S}_B)$\end{small}}, obtaining $h^t_1 \in \mathcal{H}_k$.}
       \parState{%
       Train $h^{t-1}_2\in \mathcal{H}$ for $T_2$ epochs to minimize {\color{black}\begin{small}$\rho_{\mathcal{C}}(h^t_2\circ \mathcal{S}_A,\mathcal{S}_B)-\lambda R_{\mathcal{S}_A}[h^t_1,h^t_2]$\end{small}}, obtaining $h^{t}_2 \in \mathcal{H}_k$.\\\Comment{$T_2$ provides a fixed comparison point}.}
    \EndFor
    \parState{\Return $h^t_1$ such that: {\color{black}\begin{small}$t = \underset{i \in [T_1]}{\arg\min}  \left\{ R_{\mathcal{S}_A}[h^i_1,h^i_2] \; \Big\vert \; \forall j=1,2: \rho_{\mathcal{C}}(h^i_j \circ \mathcal{S}_A,\mathcal{S}_B) \leq c \hat{\rho}^*_k\right\}$\end{small}}.} 
  \end{algorithmic}
\end{algorithm}

\subsection{Deriving an Unsupervised Variant of Hyperband using the Bound}
\label{sec:hyperband}

In order to optimize multiple hyperparameters simultaneously, we create an unsupervised variant of the Hyperband method~\citep{hyperband}. Hyperband requires the evaluation of the loss for every configuration of hyperparameters. In our case, our loss is the risk function, $R_{D_A}[h_1,y]$. Since we cannot compute the actual risk, we replace it with an approximated value of our bound 
{\color{black}
\begin{equation}\label{eq:boundEstimation}
\sup\limits_{h_2 \in \mathcal{P}_{\omega}} R_{\mathcal{S}_A}[h_1,h_2] + \rho_{\mathcal{C}}(h_1 \circ \mathcal{S}_A,\mathcal{S}_B)
\end{equation}
}
This expression differs from the original bound in two ways. First, {\color{black}we neglect the term $\mathcal{K}$ as we already explained in Sec.~\ref{sec:analyzing} that it tends to be small}. In addition, the term {\color{black}$\inf_{h \in \mathcal{P}_{\omega}} \rho_{\mathcal{C}}(h \circ \mathcal{S}_A,\mathcal{S}_B)$}
is replaced with {\color{black}$\rho_{\mathcal{C}}(h_1 \circ \mathcal{S}_A,\mathcal{S}_B)$}, which can be easily estimated. 
This term can fit as a good replacement, {\color{black} since $\inf_{h \in \mathcal{P}_{\omega}} \rho_{\mathcal{C}}(h \circ \mathcal{S}_A,\mathcal{S}_B) \leq \rho_{\mathcal{C}}(h_1 \circ \mathcal{S}_A,\mathcal{S}_B)$.  Secondly, we neglect the terms that arise from the transition from a population distribution to finite sample sets as they are constant and small for large enough $m_1$ and $m_2$. } 
In addition, we choose $\mathcal{A}_{\omega}$ to be a GAN method, where $\omega$ is a set of hyperparameters that includes: the complexity of the trained network, batch size, learning rate, etc'.

In particular, the function `run\_then\_return\_val\_loss' in the hyperband algorithm (Alg.~1 of~\cite{hyperband}), which is a plug-in function for loss evaluation, is provided with our bound from Eq.~\ref{eq:boundEstimation} after training $h_2$, as in Eq.~\ref{eq:dual}. Our variant of this function is listed in Alg.~\ref{algo:rtrvl}. It employs two additional procedures that are used to store the learned models $h_1$ and $h_2$ at a certain point in the training process and to retrieve these to continue the training for a large number of epochs. The retrieval function is simply a map between a vector of hyperparameters and a tuple of the learned networks and the number of epochs $T$ when stored.  For a new vector of hyperparameters, it returns $T=0$ and two randomly initialized networks, with architectures that are determined by the given set of hyperparameters. When a network is retrieved, it is then trained for a number of epochs that is the difference between the required number of epochs $T$, which is given by the hyperband method, and the number of epochs it was already trained, denoted by $T_{\text{last}}$.

\begin{algorithm}[t]
\caption{Unsupervised run\_then\_return\_val\_loss for hyperband}
\label{algo:rtrvl}
  \begin{algorithmic}[1]
  \Require{$\mathcal{S}_A$ and $\mathcal{S}_B$: unlabeled training sets; $\lambda$: a trade-off parameter; $T$: number of epochs; $\omega$: set of hyperparameters.}
      \parState{%
      [$h_1, h_2$, $T_{\text{last}}$] = return\_stored\_functions($\omega$)}
      \parState{%
        Train $h_1\in \mathcal{H}$ for $T-T_{\text{last}}$ epochs to minimize {\color{black}\begin{small}$\rho_{\mathcal{C}}(h_1\circ \mathcal{S}_A,\mathcal{S}_B)$\end{small}}.}
       \parState{%
       Train $h_2 \in \mathcal{H}$ for $T-T_{\text{last}}$ epochs to minimize {\color{black}\begin{small}$\rho_{\mathcal{C}}(h_2\circ \mathcal{S}_A,\mathcal{S}_B)-\lambda R_{\mathcal{S}_A}[h_1,h_2]$\end{small}}.}
    \parState{store\_functions($\omega$, [$h_1, h_2, T$]) }
    \parState{\Return {\color{black}\begin{small}$R_{\mathcal{S}_A}[h_1,h_2] + \rho_{\mathcal{C}}(h_1 \circ \mathcal{S}_A,\mathcal{S}_B)$\end{small}}.} 
  \end{algorithmic}
\end{algorithm}

\section{The non-unique case}\label{sec:dualProofExt}

In various cases there are multiple unknown target functions from $A$ to $B$, i.e., there is a set $\mathcal{T}$ of alternative target functions $y$. For instance, the domain $\mathcal{X}_A$ is a set of images of shoe edges and $\mathcal{X}_B$ is a set of images of shoes. There are multiple mappings that take edges of a shoe and return a shoe that fits these edges (each mapping colors the shoes in a different way). It is important to note that $\mathcal{T}$ contains only a subset of the alternative mappings between $A$ and $B$. For instance, in the edges to shoes example, there are mappings that take edges and return a shoes that does not fit these edges.  

As before, the cross-domain mapping algorithm is provided with access to the distributions $D_A$ and $D_B$. However, in this case, the goal of the algorithm is to return a hypothesis $h_1\in \mathcal{H}$ that is close to one of the target functions $y \in \mathcal{T}$, i.e., minimizes $\inf_{y \in \mathcal{T}} R_{D_A}[h_1,y]$.

The bound in Thm.~\ref{thm:boundGEN} can be readily extended to the non-unique case by simply taking $\inf_{y\in \mathcal{T}}$ to both sides of the inequality in Thm.~\ref{thm:boundGEN}. This results in an upper bound on $\inf_{y \in \mathcal{T}} R_{D_A}[h_1,y]$, instead of $R_{D_A}[h_1,y]$ for a specific target function $y$. 

In the non-unique case, this bound might not be tight for various $\omega\in \Omega$. For instance, since there are multiple possible target functions $y \in \mathcal{T}$ and the term {\color{black}$\sup_{h_2 \in \mathcal{P}_{\omega}} R_{\mathcal{S}_A}[h_1,h_2]$} is almost the diameter of {\color{black} $\mathcal{P}_{\omega}(\mathcal{S}_A,\mathcal{S}_B)$}, it can be large, if $h_1 \approx y_1$ and $h_2 \approx y_2$, where $y_1, y_2 \in \mathcal{T}$, such that, $y_1 \neq y_2$. However, in this case, $\inf_{y \in \mathcal{T}}R_{D_A}[h_1,y] \approx 0$. Similar to Sec.~\ref{sec:analyzing}, we extend the assumption that $\inf_{h,d}\mathcal{K}(h,d;y)$ is small to be $\inf_{y\in \mathcal{T}}\inf_{h,d} \mathcal{K}(h,d;y)$.

A tighter bound would result, if we are able to select $\omega \in \Omega$ that concentrates the members of {\color{black} $\mathcal{P}_{\omega}(\mathcal{S}_A,\mathcal{S}_B)$} around one target function $y\in \mathcal{T}$. To do so, we select $\omega$ that minimizes the bound. In other words, we select $\omega$ that minimizes the diameter of {\color{black} $\mathcal{P}_{\omega}(\mathcal{S}_A,\mathcal{S}_B)$}, such that, the {\color{black}mapping discrepancies} of the hypotheses in {\color{black} $\mathcal{P}_{\omega}(\mathcal{S}_A,\mathcal{S}_B)$} are kept small. 

\paragraph{Equivalence Classes According to a Fixed Encoder}

To deal with the problem discussed above, we take an encoder-decoder perspective of each target function $y \in \mathcal{T}$, such that, by fixing the first layers of the mapping $y$, the ambiguity between these functions vanishes, and only one possible solution remains, e.g., the function is determined by the encoder part. 

To formalize this idea, we take a hypothesis class $\mathcal{H} := \{h_{\theta,\omega} = g_{\theta} \circ f_{\omega} \mid \theta \in \Theta, \omega \in \Omega\}$ consisting of hypotheses that are parameterized by two sets of parameters $\theta \in \Theta$ and $\omega \in \Omega$. In addition, we denote $\mathcal{H}_{\omega} := \{h_{\theta,\omega} \mid \theta \in \Theta\}$. Specifically, $\mathcal{H}$ serves as a set of neural networks of a fixed architecture with $l_1+l_2$ layers. Each hypothesis $h_{\theta,\omega}$ is a neural network of an encoder-decoder architecture. The encoder, $f_{\omega}$, consists of the first $l_1$ layers and the decoder, $g_{\theta}$, consists of the last $l_2$ layers. In addition, $\omega$ and $\theta$ denote the sets of weights of $f_{\omega}$ and $g_{\theta}$ (resp.). 

We take $\mathcal{A}_{\omega}$ that returns a hypothesis $h$ from $\mathcal{H}_{\omega}$ that minimizes {\color{black} $\rho_{\mathcal{C}}(h \circ \mathcal{S}_A,\mathcal{S}_B)$} and  {\color{black} $\mathcal{P}_{\omega}(\mathcal{S}_A,\mathcal{S}_B)$} is defined accordingly.

{\color{black} Under this characterization, $h_1$ and $h_2$ are two autoencoders of the same architecture with shared parameters $\omega$ (encoder) and un-shared parameters $\theta_1$ and $\theta_2$ (decoder). This leads to an extended version of Alg.~\ref{algo:when}. Informally, $\theta_1$ (decoder of $h_1$) is trained to minimize $\rho_{\mathcal{C}}(h_1 \circ \mathcal{S}_A,\mathcal{S}_B)$ as we would like to have $h_1 \in \mathcal{P}_{\omega}(\mathcal{S}_A,\mathcal{S}_B)$. In addition, $\theta_2$ is trained to maximize $R_{\mathcal{S}_A}[h_1,h_2]$ and to minimize $\rho_{\mathcal{C}}(h_2 \circ \mathcal{S}_A,\mathcal{S}_B)$ as we would like to find $h_2 \in \mathcal{P}_{\omega}(\mathcal{S}_A,\mathcal{S}_B)$ that maximizes $R_{\mathcal{S}_A}[h_1,h_2]$. Finally, the parameters $\omega$ (the shared encoder) are trained to minimize the bound as we would like to choose $\omega$ that provides a minimal value to the bound. }

\begin{algorithm}[t]
\caption{Deciding when to stop training $h_1$ (non-unique case)}
\label{algo:when2}
  \begin{algorithmic}[1]
  \Require{$\mathcal{S}_A$ and $\mathcal{S}_B$: unlabeled training sets; $\mathcal{H}$: a hypothesis class; $\epsilon_0$: a threshold; $\lambda$: a trade-off parameter; $T_0$: a fixed number of epochs for $\omega$; $T_1$: a fixed number of epochs for $h_1$; $T_2$: a fixed number of epochs for $h_2$.}
    \State Initialize the shared parameters $\omega_0 \in \Omega$ at random.
    \State Initialize the parameters $\theta_{1,0},\theta_{2,0} \in \Theta$ of $h_1$ and $h_2$ (resp.) at random.
    \For {$t = 1,\dots,T_0$}
      \parState{%
        Train $\omega_{t-1}$ for one epoch to minimize {\color{black}\begin{small}$R_{\mathcal{S}_A}[h^{t-1}_1,h^{t-1}_2] + \rho_{\mathcal{C}}(h^{t-1}_1\circ \mathcal{S}_A,\mathcal{S}_B)$\end{small}}, obtaining $\omega_t \in \Omega$.}
      \parState{%
        Train $\theta_{1,t-1} \in \Theta$ for $T_1$ epochs to minimize {\color{black}\begin{small}$\rho_{\mathcal{C}}(h^t_1\circ \mathcal{S}_A,\mathcal{S}_B)$\end{small}}, obtaining $\theta_{1,t} \in \Theta$.}
       \parState{%
       Train $\theta_{2,t-1} \in \Theta$ for $T_2$ epochs to minimize {\color{black}\begin{small}$\rho_{\mathcal{C}}(h^t_2\circ \mathcal{S}_A,\mathcal{S}_B)-\lambda R_{\mathcal{S}_A}[h^t_1,h^t_2]$\end{small}}, obtaining $\theta_{2,t} \in \Theta$.\\\Comment{Here, $h^t_i := g_{\theta_{i,t-1}}\circ f_{\omega_{t-1}}$ for $i=1,2$.}}
    \EndFor
    \parState{%
    Define {\color{black}\begin{small}$t := \underset{i \in [T_0]}{\arg\min}  \left\{ R_{\mathcal{S}_A}[h^i_1,h^i_2] + \rho_{\mathcal{C}}(h^{i}_1\circ \mathcal{S}_A,\mathcal{S}_B)  \right\}$\end{small}}.
    }
    \parState{\Return $h^t_1$.} 
  \end{algorithmic}
\end{algorithm}

\section{Experiments}
\label{sec:experiments}

The first group of experiments is intended to test the validity of the two predictions made in Sec.~\ref{sec:conseq}. The next group of experiments are dedicated to Algs.~\ref{algo:when},~\ref{algo:rtrvl} and~\ref{algo:when2}. 

Note that while we develop the theoretical results in the context of WGANs and IPMs, the methods developed are widely applicable. In order to demonstrate this, we run our experiments using a wide variety of GAN variants, including CycleGAN~\citep{CycleGAN2017}, DiscoGAN~\citep{pmlr-v70-kim17a}, DistanceGAN~\citep{distgan}, and UNIT~\citep{liu2017unsupervised}, as well as using WGAN itself. The choice of relying on WGAN in the analysis was made, since it is highly accepted as an effective GAN method and since it is amenable to analysis.

\subsection{Empirical Validation of The Predictions}

\paragraph{Prediction~\ref{pred1}} This prediction states that since the unsupervised mapping methods are aimed at learning minimal complexity low discrepancy functions, GANs are sufficient. In this paper, we mainly focus on a complexity measure that is the minimal number of layers of a neural network that are required in order to compute $h$, where each hidden layer is of size $\in [r_1,r_2]$. In the literature~\citep{CycleGAN2017,pmlr-v70-kim17a}, learning a mapping $h:\mathcal{X}_A\rightarrow \mathcal{X}_B$, based only on the GAN constraint on $B$  (e.g., minimize $\rho_{\mathcal{C}}(h \circ \mathcal{S}_A,\mathcal{S}_B)$), is presented as a failing baseline. In~\citep{dualgan}, among many non-semantic mappings obtained by the GAN baseline, one can find images of GANs that are successful. However, this goes unnoticed.

In order to validate the prediction that a purely GAN based solution is viable, we conducted a series of experiments using the DiscoGAN architecture and GAN loss in the target distribution only. We consider image domains $A$ and $B$, where $\mathcal{X}_A = \mathcal{X}_B = \mathbb{R}^{3\times 64\times 64}$.  

In DiscoGAN, the generator is built of: (i) an encoder consisting of convolutional layers with $4\times 4$ filters, followed by Leaky ReLU activation units and (ii) a decoder consisting of deconvolutional layers with $4\times 4$ filters, followed by a ReLU activation units. Sigmoid is used for the output layer. Between four to five convolutional/deconvolutional layers are used, depending on the domains used in $A$ and $B$ (we match the published code architecture per dataset). The discriminator is similar to the encoder, but has an additional convolutional layer as the first layer and a sigmoid output unit. 

The first set of experiments considers the CelebA face dataset. The results for hair color conversions are shown in  Fig.~\ref{fig:celebA_blond}. It is evident that the output image is closely related to the input images, despite the fact that cycle loss terms were not used.

{\color{black}To quantitatively validate the prediction, we trained a mapping $h$ from $A$ to $B$ using DiscoGAN with and without circularity losses and measured the expected VGG similarity between its input and output, i.e., $\mathbb{E}_{x \sim D_A}[cs(f(x),f(h(x)))]$. The VGG similarity between two images $x_1$ and $x_2$ computes $cs(f(x_1),f(x_2))$, where $cs$ is the cosine similarity function and $f(x)$ is a deep layer in the VGG network. Since the VGG network was trained to identify the content of a wide variety of classes, these vector representations, $f(x)$, are treated as compressed content descriptors of the images and $cs(f(x_1),f(x_2))$ as the degree of content similarity between the images $x_1$ and $x_2$. In Tab.~\ref{tab:var_layers_value}, we report the best VGG similarity and discrepancy values among an extensive parameter search, when trying to change the learning rate (between $10^{-5}$ to 1), the number of kernels per layer (between 10 and 300), and the weight between circularity losses and the GANs (between $10^{-5}$ and 1). As can be seen in Tab.~\ref{tab:var_layers_value}, the results (e.g., discrepancy and VGG similarity scores) of DiscoGAN with and without the circularity losses are fairly similar when varying the number of layers. }


\paragraph{Prediction~\ref{pred2}} We claim that the selection of the right number of layers $k$ is crucial in unsupervised learning. Using fewer layers than needed, will not support {\color{black}a small mapping discrepancy, i.e., $\inf_{h \in \mathcal{P}_k}\rho_{\mathcal{C}}(h\circ \mathcal{S}_A,\mathcal{S}_B)$ is large}. By contrast, adding superfluous layers would mean that {\color{black}there exist many alternative functions in $\mathcal{H}_k$ that map between the two domains, i.e., $\sup_{h_2 \in \mathcal{P}_k} R_{\mathcal{S}_A}[h_1,h_2]$ is large}. 

To see the influence of the number of layers of the generator $h_1$ on the results, we employed the DiscoGAN~\citep{pmlr-v70-kim17a} public implementation and added or removed layers from the generator.
The experiment was done on the CelebA dataset where 8 layers are employed in the experiments of~\citep{pmlr-v70-kim17a}. 

The results for male to female conversion are illustrated in Fig.~\ref{fig:sw000}. Note that since the encoder and the decoder parts of the learned network are symmetrical, the number of layers is always even. As can be seen visually, changing the number of layers has a dramatic effect on the results. The best results are obtained at 6 or 8 layers with 6 having the best alignment and 8 having better discrepancy. The results degrade quickly, as one deviates from the optimal value. Using fewer layers, the GAN fails to produce images of the desired class. Adding layers, the semantic alignment is lost, just as expected. {\color{black}The experiment is repeated for both CycleGAN~\citep{CycleGAN2017} and WGAN~\citep{bojanowski2017optimizing} for other datasets in Figs~\ref{fig:sw001}-\ref{fig:sw003}.}

{\color{black} As can be seen in Tab.~\ref{tab:var_layers_value}, when varying the number of layers, both the discrepancy and the VGG similarity decrease, with or without circularity losses. It is not surprising that the discrepancy decreases, since when increasing the complexity of the network, it can capture the target distribution better. However, the decreasing value of the VGG similarity indicates that the alignment is lost, as the mapping generates images that are not similar to the input images. This validates Pred.~\ref{pred2} as we can see that the number of layers has a dramatic effect on the results.}

{\color{black} While our discrete notion of complexity seems to be highly related to the quality of the results, the norm of the weights do not seem to point to a clear architecture, as shown in Tab.~\ref{tab:a001}(a). Since the table compares the norms of architectures of different sizes, we also approximated the functions using networks of a fixed depth $k=18$ and then measured the norm. These results are presented in Tab.~\ref{tab:a001}(b). In both cases, the optimal depth, which is 6 or 8, does not appear to have a be an optimum in any of the measurements. 
}



\begin{figure}[t]
    \centering
    \begin{minipage}{.5\textwidth}
        \centering
        \includegraphics[width=0.95\linewidth]{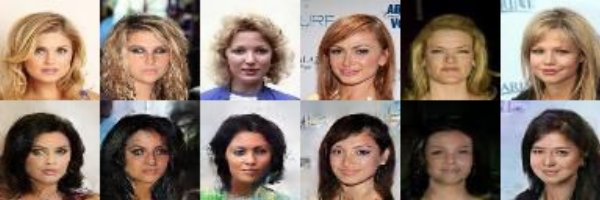}
    \end{minipage}%
    \begin{minipage}{0.5\textwidth}
        \centering
        \includegraphics[width=0.95\linewidth]{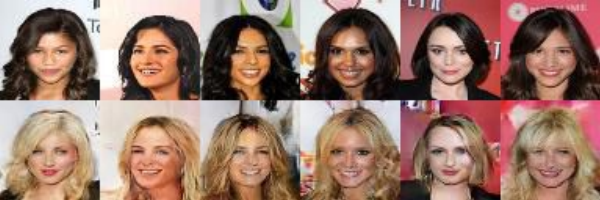}
    \end{minipage}
    \caption{\label{fig:celebA_blond} Results for the celebA dataset for converting blond to black hair and vice versa, when the mapping is obtained by the GAN loss without additional losses.}
\end{figure}

\begin{figure}[t]
  \centering
 \includegraphics[width=0.6\linewidth,clip]{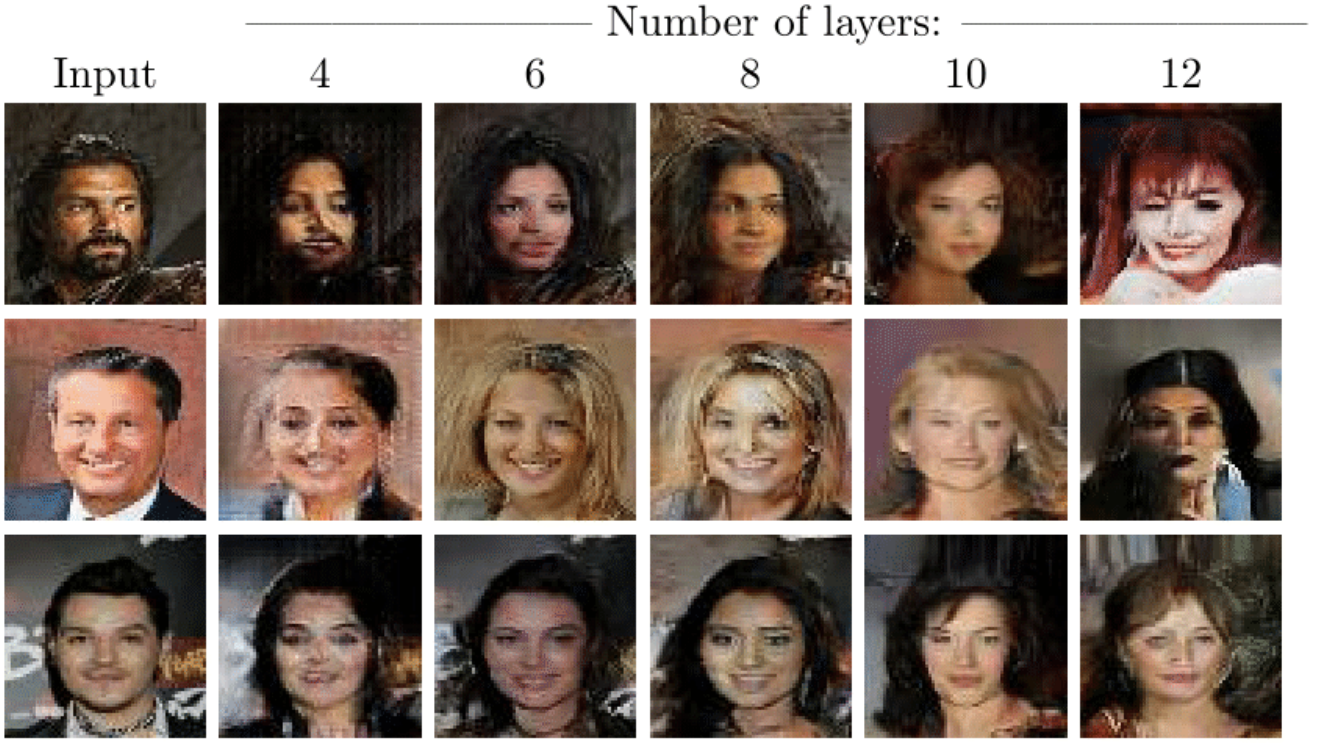}
  \caption{\label{fig:sw000} {\bf Results of varying the number of layers of the generator.} Results of DiscoGAN on CelebA Male to Female transfer. The best results are obtained for $6$ or $8$ layers. For more than $6$ layers, the alignment is lost.} 
\end{figure}

\subsection{Results for Algs.~\ref{algo:when},~\ref{algo:rtrvl} and~\ref{algo:when2}}

\begin{figure*}[t]
\centering
\hspace{-0.6cm}
  \begin{tabular}{c@{~}c@{~}c}
  \includegraphics[width=0.4\linewidth, clip]{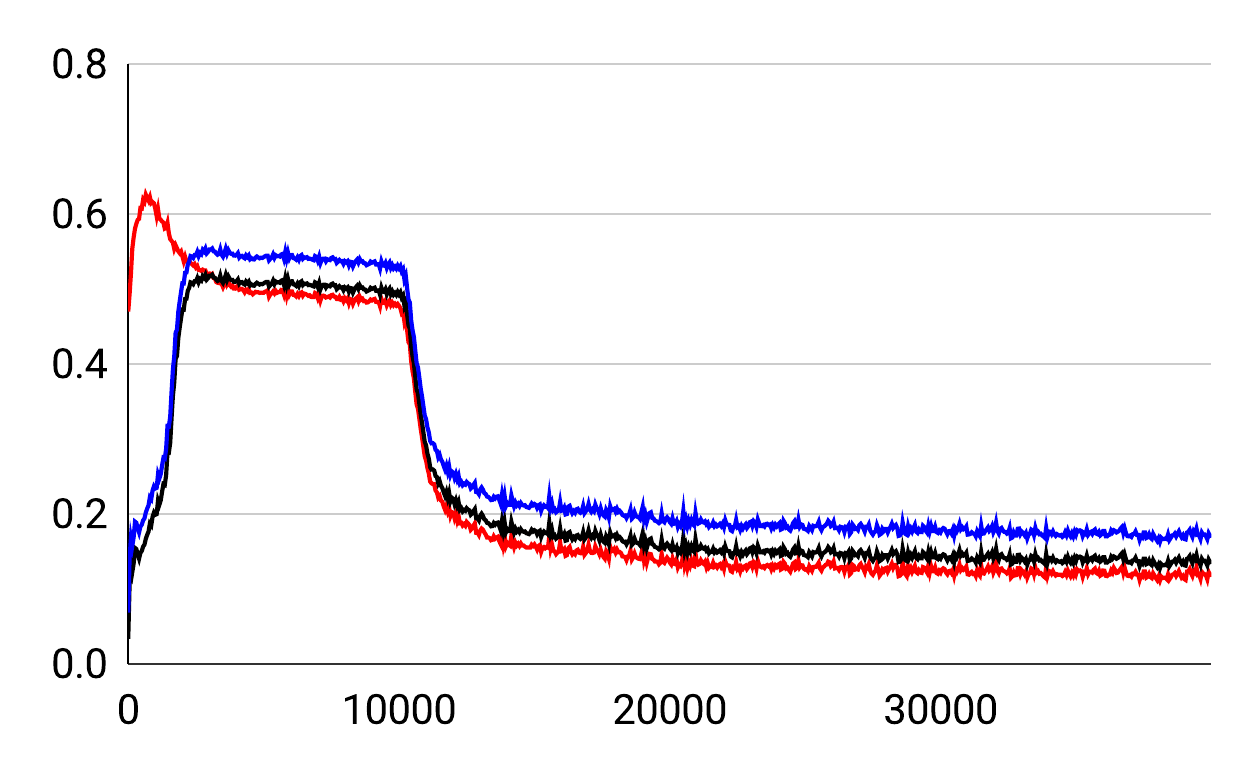} &
\includegraphics[width=0.4\linewidth, clip]{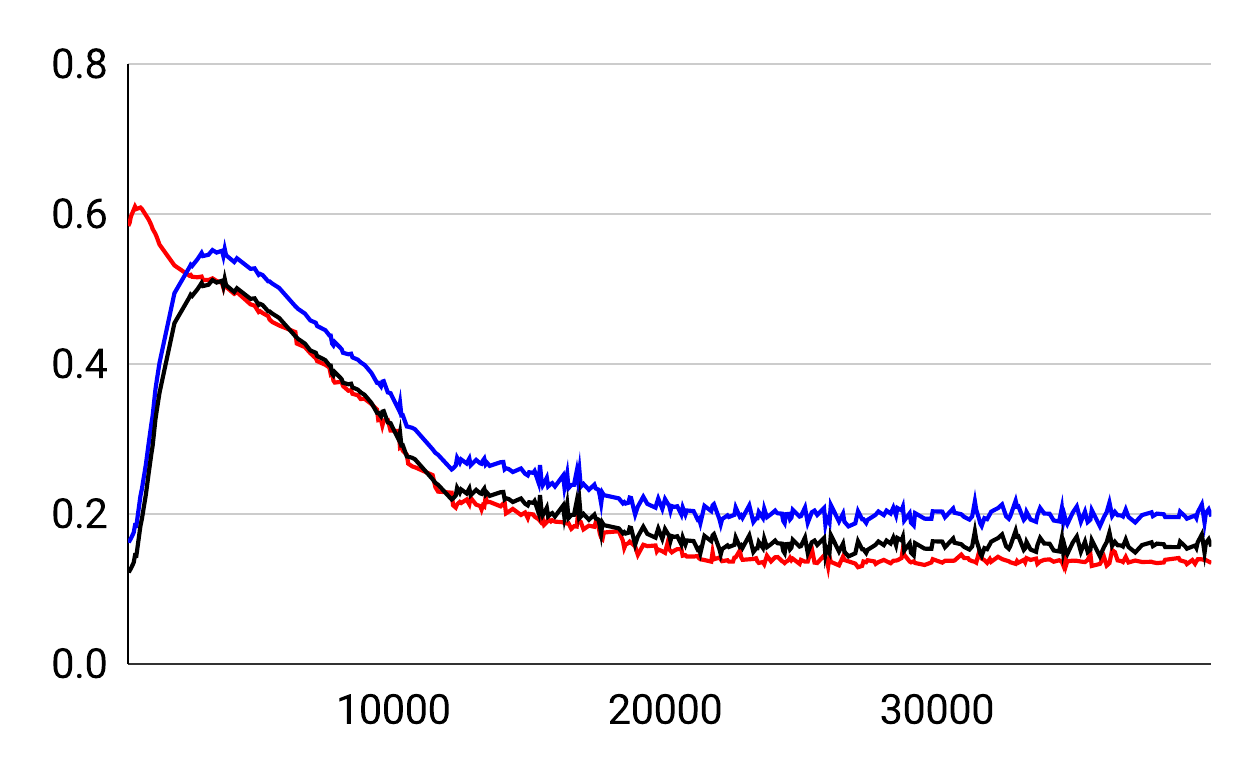}\\
(DiscoGAN, Handbags2Edges)  & (DistanceGAN, Shoes2Edges) \\
  \includegraphics[width=0.4\linewidth, clip]{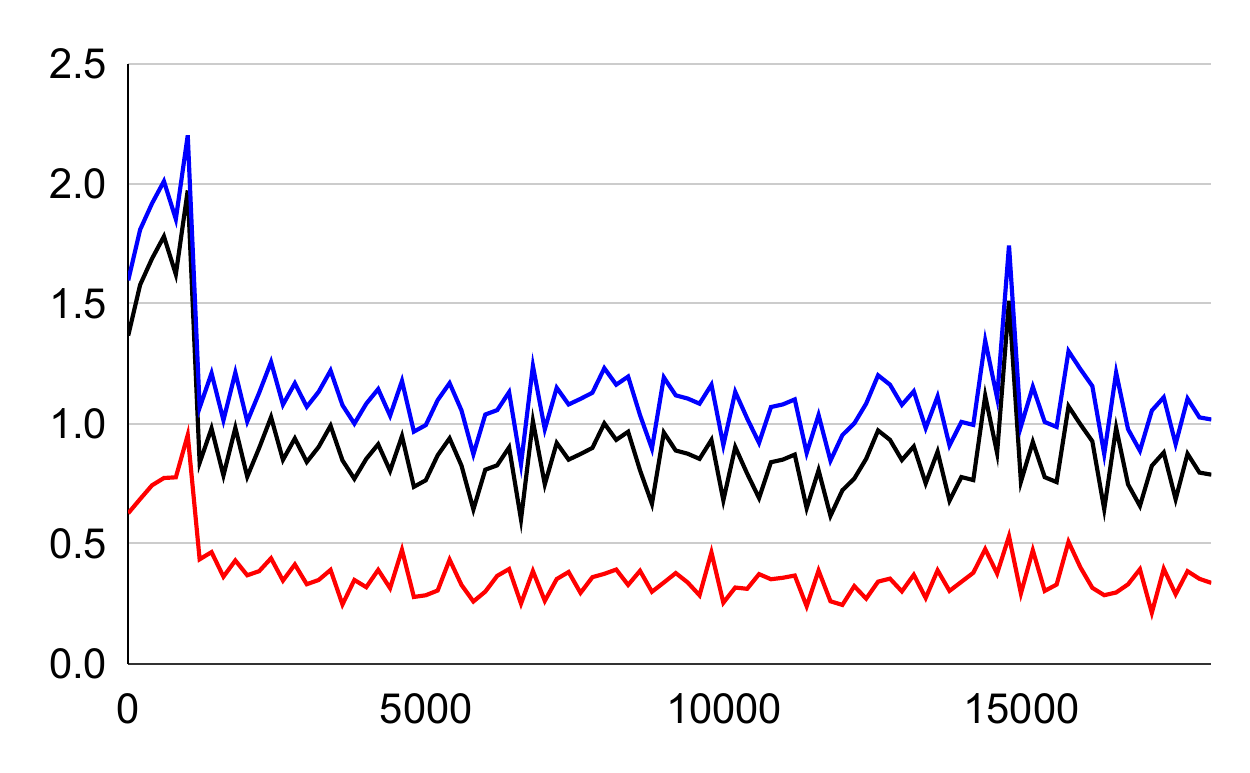}&
\includegraphics[width=0.4\linewidth, clip]{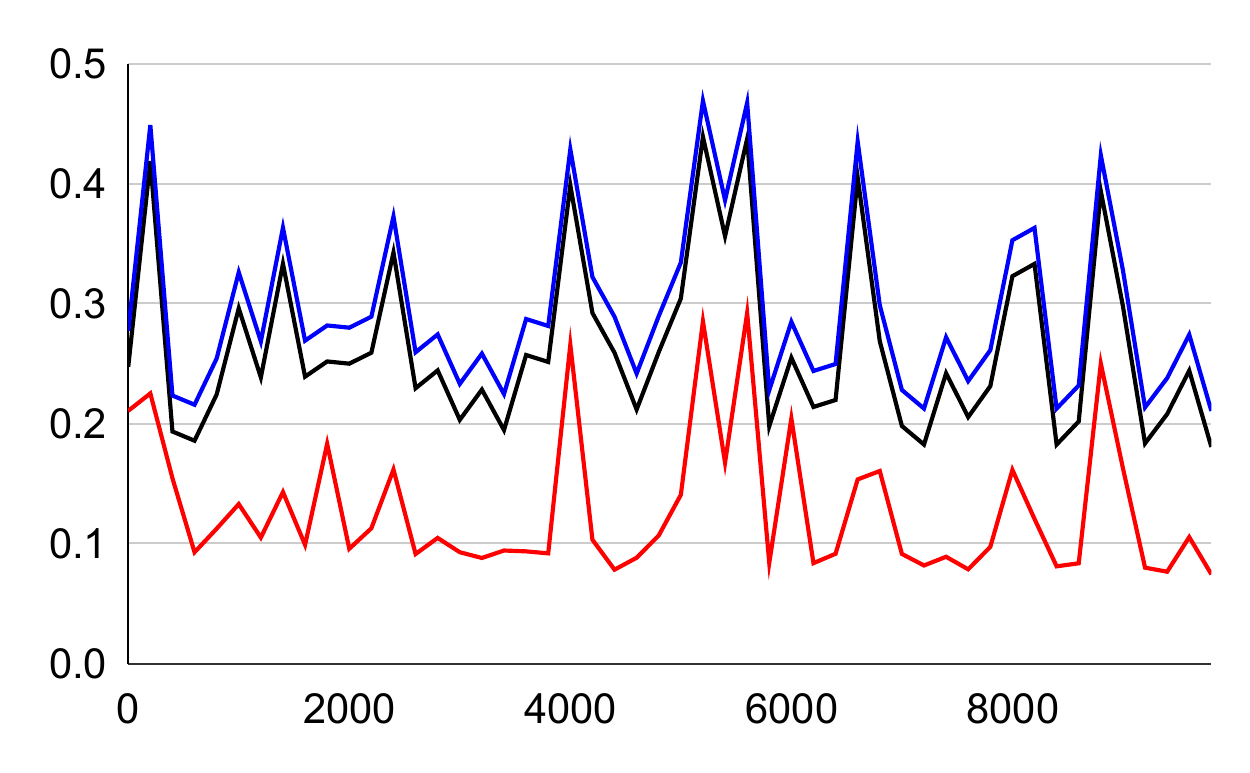}\\
(CycleGAN, Cityscapes)  & (CycleGAN, Maps) \\
\end{tabular}
\caption{{\bf Results of Alg.~\ref{algo:when}.} Ground truth errors {\color{black}$R_{D_A}[h^t_1,y]$} are in red, {\color{black}$R_{\mathcal{S}_A}[h^t_1,h^t_2] + c\hat{\rho}^*_k$} in black and {\color{black}$R_{\mathcal{S}_A}[h^t_1,h^t_2] + c\hat{\rho}^*_k + \mathcal{K}$ are in blue ($c=2$)}. x-axis is the iteration. y-axis is the expected risk/bound. It takes a few epochs for $h_1$ to have a small enough discrepancy, until which the bound is ineffective.}
  \label{fig:whenMain}
\end{figure*}

We test the three algorithms on three unsupervised alignment methods: DiscoGAN~\citep{pmlr-v70-kim17a}, CycleGAN~\citep{CycleGAN2017}, and DistanceGAN~\citep{distgan}. In DiscoGAN {\color{black}and CycleGAN}, we train $h_1$ (and $h_2$), using two GANs and two circularity constraints; in DistanceGAN, to train $h_1$ (and $h_2$), one GAN and one distance correlation loss are used. The published hyperparameters for each dataset are used, except when using Hyperband, where we vary the number of layers, the learning rate and the batch size. 

Five datasets were used in the experiments: (i) aerial photographs to maps, trained on data scraped from Google Maps~\citep{pix2pix}, (ii) the mapping between photographs from the cityscapes dataset and their per-pixel semantic labels~\citep{Cordts2016Cityscapes}, (iii) architectural photographs to their labels from the CMP Facades dataset~\citep{facades}, (iv) handbag images~\citep{handbag} to their binary edge images, as obtained from the HED edge detector~\citep{hed}, and (v) a similar dataset for the shoe images from~\citep{shoes}.

Throughout the experiments of Alg.~\ref{algo:when}, fixed values are used as the {\color{black}tolerance hyperparameter ($c=2$)}. The tradeoff parameter between the dissimilarity term and the fitting term during the training of $h_2$ is set, per dataset, to be the maximal value, such that the fitting of $h_2$ provides a solution that has a discrepancy lower than the threshold. This is done once, for the default parameters of $h_1$, as given in the original DiscoGAN and DistanceGAN~\citep{pmlr-v70-kim17a,distgan}.

\begin{table*}[t]
\centering
\begin{small}
\caption{Pearson correlations and the corresponding p-values (in parentheses) of the ground truth error with: (i) the bound, (ii) the GAN losses, and (iii) the circularity losses/distance similarity loss. 
}
\label{tab:mainresults}
\resizebox{0.99000995\textwidth}{!}{
\begin{tabular}{l@{~~}l@{~~}llccc}
\toprule
Method & Dataset & Bound & $GAN_A$ & $GAN_B$ & $Cycle_A/Dist$ & $Cycle_B$ \\
\midrule

Disco-    & Shoes2Edges    & {\bf 1.00} (\textless1E-16)          & -0.15 (3E-03)         & -0.28 (1E-08)  & 0.76(\textless1E-16)           & 0.79(\textless1E-16)   \\
GAN~\cite{pmlr-v70-kim17a}          &&&&&\\ 
& Bags2Edges &        {\bf 1.00}	(\textless1E-16)& -0.26 (6E-11) &	-0.57 (\textless1E-16) &	0.85 (\textless1E-16) &	0.84 (\textless1E-16)           \\
& Cityscapes        & {\bf 0.94}   (\textless1E-16)        & -0.66 (\textless1E-16)          & -0.69 (\textless1E-16)          & -0.26  (1E-07)        & 0.80 (\textless1E-16)   \\
 &  Facades           & {\bf 0.85}  (\textless1E-16)         & -0.46  (\textless1E-16)        & 0.66 (\textless1E-16)           & 0.92 (\textless1E-16)          & 0.66  (\textless1E-16)         \\
 &   Maps              & {\bf 1.00}     (\textless1E-16)      & -0.81 (\textless1E-16)          & 0.58  (\textless1E-16)         & 0.20   (9E-05)        & -0.14   (5E-03)       \\
Distance- & Shoes2Edges    &           {\bf 0.98} (\textless1E-16)	& - &	-0.25	(2E-16) & -0.14 (1E-05)	    &       -         \\
GAN~\cite{distgan}     &&&&&\\      
 & Bags2Edges &            {\bf 0.93} (\textless1E-16) & - &	-0.08 (2E-02) &	0.34	(\textless1E-16)     &        -        \\
 
           & Cityscapes        & {\bf 0.59}  (\textless1E-16)         & -          & 0.22   (1E-11)        & -0.41   (\textless1E-16)        &          -      \\
          & Facades           & {\bf 0.48}    (\textless1E-16)       & -        & 0.03  (5E-01)         & -0.01   (9E-01)       &        -        \\
 &  Maps              & {\bf 1.00}    (\textless1E-16)       & -         & -0.73  (\textless1E-16)         & 0.39   (4E-16)        &      -          \\
{\color{black} Cycle-} & {\color{black}Shoes2Edges} & {\color{black}{\bf 0.99} (<1E-16)} & {\color{black}0.44 (5E-10)} & {\color{black}0.038 (3E-12)} & {\color{black}-0.44 (5E-13)} & {\color{black}-0.40 (3E-11)} \\
{\color{black} GAN~\cite{CycleGAN2017} } \\
& {\color{black} Bags2Edges}    &          {\color{black} {\bf 0.99} (\textless1E-16)}	& {\color{black} -0.23 (\textless1E-16)} &	{\color{black} 0.21 (\textless2E-14)} & {\color{black} -0.20 (5E-15)} & {\color{black} -0.34 (4E-10)}  \\
& {\color{black}Cityscapes} & {\color{black}{\bf 0.91} (<1E-16)} & {\color{black}0.30 (6E-11)} & {\color{black}0.024 (4E-11)} & {\color{black}0.37 (3E-05)} & {\color{black}0.42 (2E-14)} \\
& {\color{black} Facades}    &          {\color{black} {\bf 0.73} (\textless1E-16)}	& {\color{black} -0.02 (\textless1E-16)} &	{\color{black} -0.1 (\textless1E-16)} & {\color{black} -0.14 (4E-10)} & {\color{black} 0.2 (3E-11)}       \\
& {\color{black} Maps}   & {\color{black} {\bf 0.85}  (\textless1E-16)}         &  {\color{black} 0.01 (5E-16)  } & {\color{black} 0.26 (3E-16) } & {\color{black} -0.39 (1E-15)  }         &      {\color{black}    -0.32 (4E-10)   }   \\

\bottomrule
\end{tabular}}
\end{small}
\end{table*}

\paragraph{Stopping Criterion (Alg.~\ref{algo:when})} For testing the stopping criterion suggested in Alg.~\ref{algo:when}, we compared, at each time point, three scores. {\color{black} The first, $R_{\mathcal{S}_A}[h_1,h_2] + c\hat{\rho}^*_k + \inf_{h \in \mathcal{P}_k}\inf_{d \in \mathcal{C},~\beta(d)\leq 1} \mathcal{K}(h,d;y)$, is our bound in Thm.~\ref{thm:boundGEN}, with $\rho^*_k := \inf_{h \in \mathcal{P}_k} \rho_{\mathcal{C}}(h \circ \mathcal{S}_A,\mathcal{S}_B)$ replaced with its upper bound $\hat{\rho}^*_k$ (see Sec.~\ref{sec:est}) and neglecting the generalization gap terms. The second, $R_{\mathcal{S}_A}[h_1,h_2] + c\hat{\rho}^*_k$, is our bound, excluding the term $\inf_{h,d} \mathcal{K}(h,d;y)$. The third is ground truth error, $R_{D_A}[h_1,y]$, where $y$ is the ground truth mapping that matches $x$ in domain $B$. The expectation in the ground truth error is taken with respect to the test dataset. The estimation of $\inf_{h,d} \mathcal{K}(h,d;y)$ is described in Sec.~\ref{sec:estK}.} 

The results are depicted in the main results table (Tab.~\ref{tab:mainresults}) as well as in Fig.~\ref{fig:whenMain} for DiscoGAN, DistanceGAN and {\color{black}CycleGAN}. 

Tab.~\ref{tab:mainresults} presents the correlation and p-value between the ground truth error, as a function of the training iteration, and the bound. A high correlation (low p-value) between the bound and the ground truth error, as a function of the iteration, indicates the validity of the bound and the utility of the algorithm. Similar correlations are shown with the GAN losses and the reconstruction losses (DiscoGAN and CycleGAN) or the distance correlation loss (DistanceGAN), in order to demonstrate that these are much less correlated with the ground truth error. In Fig.~\ref{fig:whenMain}, we omit the other scores in order to reduce clutter.

As can be seen, there is an excellent match between the mean ground truth error of the learned mapping $h_1$ and the predicted error. No such level of correlation is present when considering the GAN losses or the reconstruction losses (for DiscoGAN and CycleGAN), or the distance correlation loss of DistanceGAN. Specifically, the very low p-values in the first column of Tab.~\ref{tab:mainresults} show that there is a clear correlation between the ground truth error and our bound for all datasets and methods. For the other columns, the values in question are chosen to be the losses used for $h_1$. The lower scores in these columns show that none of these values are as correlated with the ground truth error, and so cannot be used to estimate this error.

In the experiment of Alg.~\ref{algo:when} for DiscoGAN, which has a large number of sample points, the cycle from $B$ to $A$ and back to $B$ is significantly correlated with the ground truth error with very low p-values in four out of five datasets. However, its correlation is significantly lower than that of our bound.

These results also demonstrate the tightness of the bound. As can be seen, {\color{black} the bound is always highly correlated with the test error and in most cases, it is tight as well (close to the test error). Bounds that are highly correlated with the test error are very useful as they faithfully indicate when the test error is smaller. }

\paragraph{Selecting Architecture with the Modified Hyperband Algorithm (Alg.~\ref{algo:rtrvl})} Our bound is used in Sec.~\ref{sec:hyperband} to create an unsupervised variant of the Hyperband method. In addition to selecting the architecture, this allows for the optimization of multiple hyperparameters at once, while enjoying the efficient search strategy of the Hyperband method~\citep{hyperband}. 

Fig.~\ref{fig:hyperband} demonstrates the applicability of our unsupervised Hyperband-based method for different datasets, employing both DiscoGAN and DistanceGAN. The graphs show the error and the bound obtained for the selected configuration after up to 35 Hyperband iterations. As can be seen, in all cases, the method is able to recover a configuration that is significantly better than what is recovered, when only optimizing for the number of layers. To further demonstrate the generality of our method, we applied it on the UNIT~\citep{liu2017unsupervised} architecture. Specifically, for DiscoGAN and DistanceGAN, we optimize the number of encoder and decoder layers, batch size and learning rate, while for UNIT, we optimize for the number of encoder and decoder layers, number of resnet layers and learning rate. Fig.~\ref{fig:hyperband_unit} and Tab.~\ref{fig:hyperband}(b) show the convergence on the Hyperband method. 

\begin{figure*}[t]
\begin{tabular}{c@{~}c}
\multicolumn{2}{c}{\includegraphics[width=1.0\linewidth, clip]{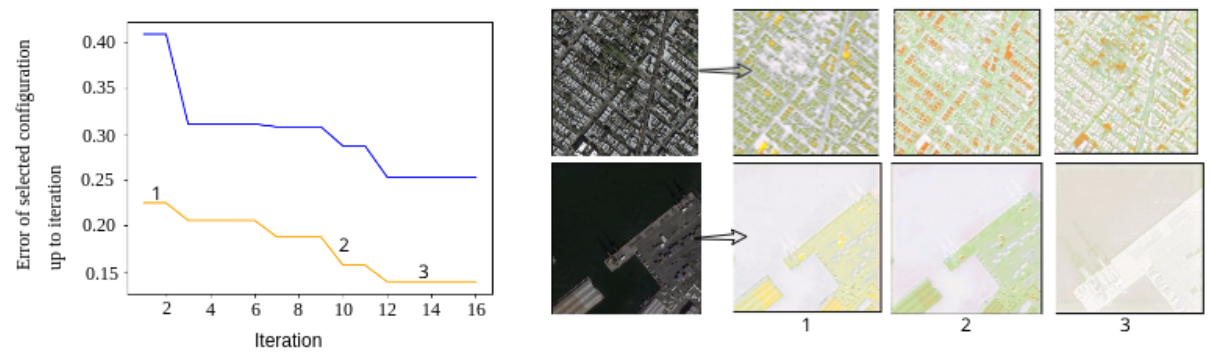}}\\
~~~~~~~~~~~~~~~~~~~~~~~~~~~~~~~~(a) & ~~~~(b)\\
\end{tabular}
\caption{Applying unsupervised Hyperband for selecting the best configuration for UNIT for the Maps dataset. The $x$-axis is the epoch count and the $y$-axis is the error/bound value of the selected configuration. (a) blue and orange lines are bound and ground truth error as in Fig.~\ref{fig:hyperband}. (b) Images produced for three different configurations, as indicated on the plot in (a).} 
 \label{fig:hyperband_unit}
\end{figure*}

\begin{figure*}[t]
\centering
 
\vspace{-2cm}
\begin{minipage}{0.656772\textwidth}
\centering
\begin{tabular}{c@{~}c@{~}c}
\centering{\begin{small}\begin{turn}{90}\;\;\;\;\;\; \hspace{3mm} Maps\end{turn}\end{small}}&
\includegraphics[width=0.482\linewidth, clip]{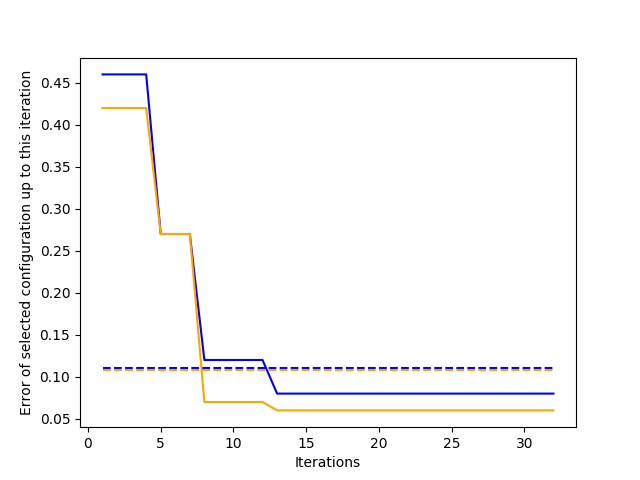}\vspace{-2mm}&
\includegraphics[width=0.482\linewidth, clip]{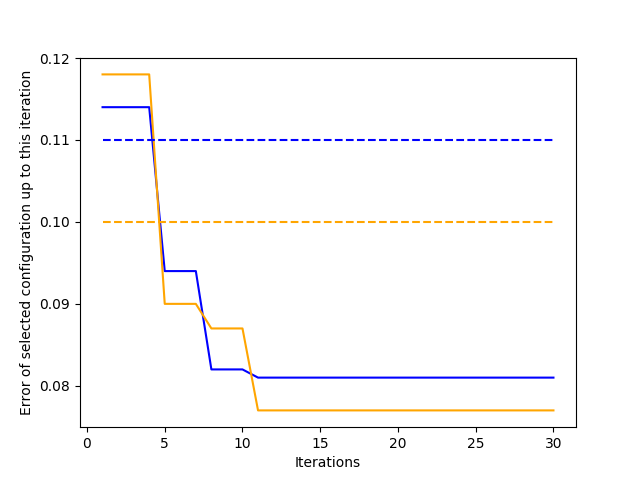}\vspace{-2mm}\\
\centering{\begin{small}\begin{turn}{90}\;\;\;\;\;\; \hspace{1mm}Cityscapes\end{turn}\end{small}}&
\includegraphics[width=0.482\linewidth, clip]{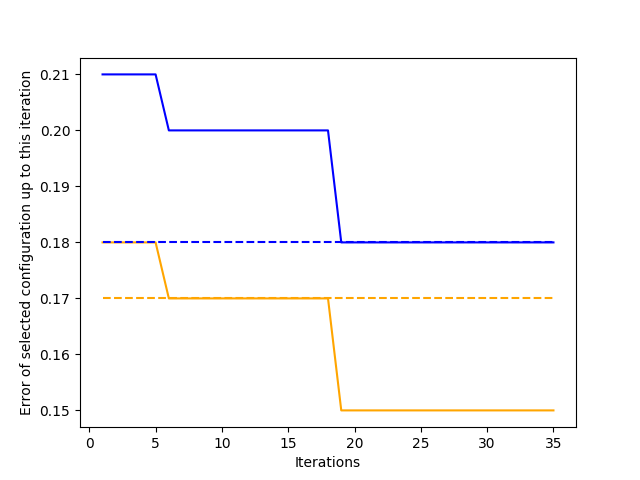}\vspace{-2mm}&
\includegraphics[width=0.482\linewidth, clip]{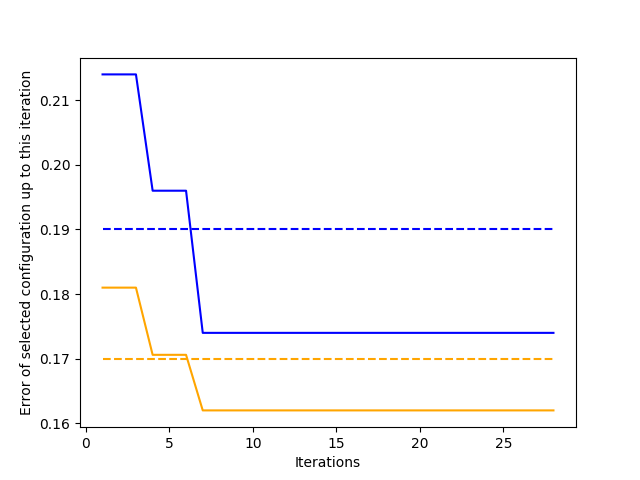}\vspace{-2mm}\\

\centering{\begin{small}\begin{turn}{90}\;\;\;\;\;\; \hspace{3mm}Facades\end{turn}\end{small}}&
\includegraphics[width=0.482\linewidth, clip]{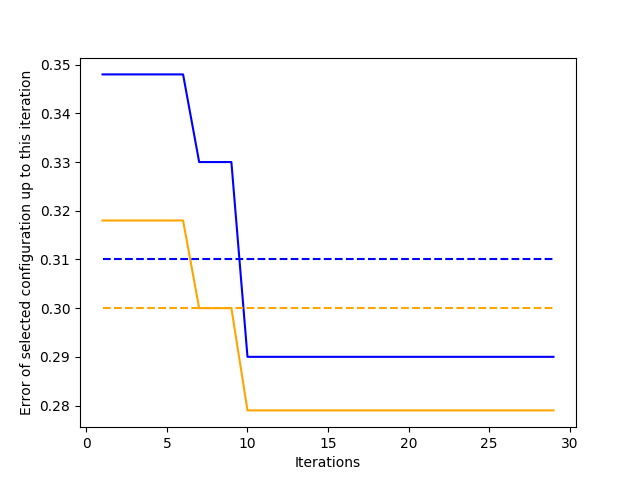}\vspace{-2mm}&
\includegraphics[width=0.482\linewidth, clip]{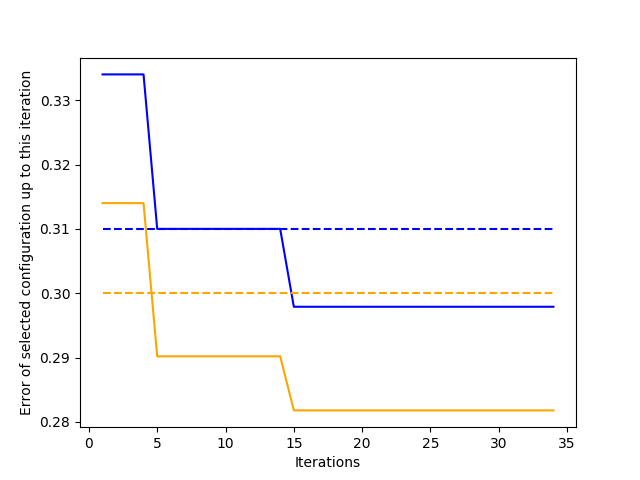}\vspace{-2mm}\\
\centering{\begin{small}\begin{turn}{90}\;\;\;\;\;\; \hspace{-3mm} Bags2Edges\end{turn}\end{small}}&
\includegraphics[width=0.482\linewidth, clip]{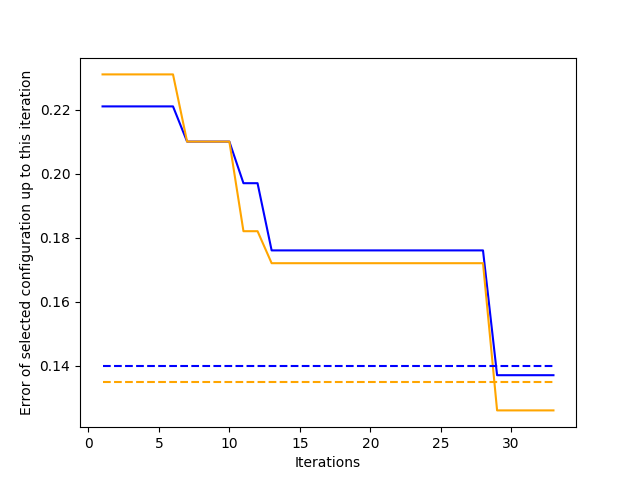}\vspace{-2mm}&
\includegraphics[width=0.482\linewidth, clip]{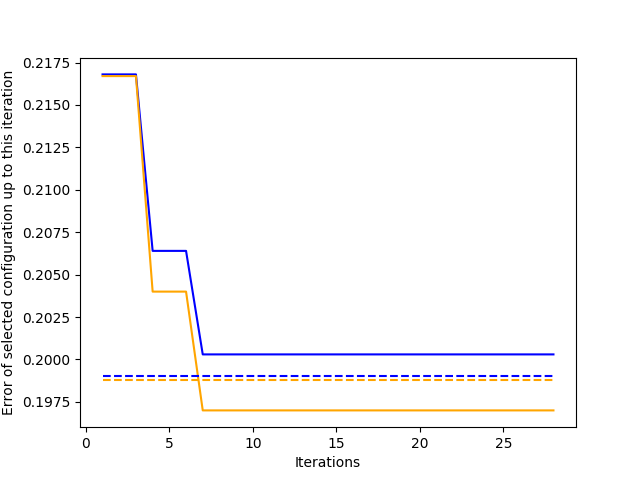}\vspace{-2mm}\\

\centering{\begin{small}\begin{turn}{90}\;\;\;\;\;\; Shoes2Edges\end{turn}\end{small}}&
\includegraphics[width=0.482\linewidth, clip]{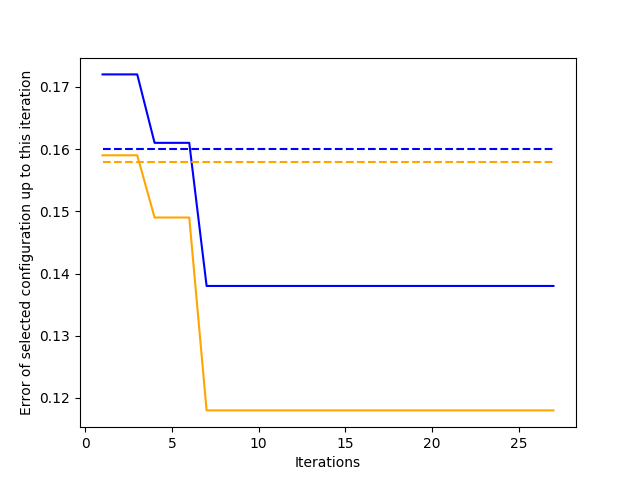}\vspace{-2mm}&
\includegraphics[width=0.482\linewidth, clip]{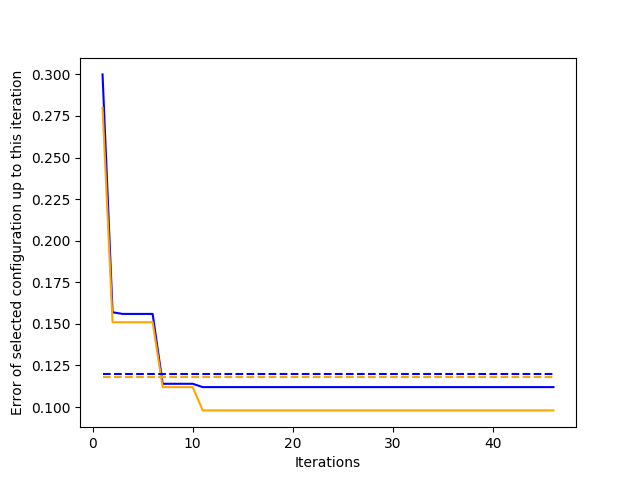}\vspace{-2mm}\\
 \multicolumn{3}{c}{(a)}
  \\
\end{tabular}
  \end{minipage}%
\begin{minipage}{.71\textwidth}
\hspace{0.1mm}
\begin{small}
\resizebox{0.48\textwidth}{!}{
\begin{tabular}{l@{~~}ccc}
\toprule
Dataset & Number & Batch & Learning \\
 &  Layers & Size & Rate \\
\midrule
\multicolumn{4}{c}{DiscoGAN~\citep{pmlr-v70-kim17a}}\\
            Shoes2Edges       & 3 & 24 & 0.0008 \\
            Bags2Edges        & 2 & 59 &	0.0010 \\
                                  Cityscapes        & 3 & 27 & 0.0009 \\
                                  Facades           & 3 & 20 & 0.0008 \\
                                  Maps              & 3 & 20 & 0.0005 \\
                                  \multicolumn{4}{c}{DistanceGAN~\citep{distgan}} \\
                             Shoes2Edges       & 3 & 15 & 0.0007 \\
                Bags2Edges        & 3 & 33 &	0.0007 \\
                                  Cityscapes        & 4 & 21 & 0.0006 \\
                                  Facades           & 3 & 8  & 0.0006 \\
                                  Maps              & 3 & 20 & 0.0005 \\     
                                  \toprule
Dataset & \#Layers & \#Res & L.Rate \\
 \midrule
\multicolumn{4}{c}{UNIT~\citep{liu2017unsupervised}} \\
                                  Maps              & 3 & 1 & 0.0003 \\     
\bottomrule
\multicolumn{4}{c}{(b)}
\end{tabular}}

\centering
\hspace{0.3mm}
\begin{tabular}{l@{~~}l@{~~~}l}
& {\scriptsize default} & {\scriptsize unsupervised}\\
 & {\scriptsize parameters} & {\scriptsize hyperband} \\
~~~$x$ & {\footnotesize $h_1(x)$} & {\footnotesize $h_1(x)$} \\

\includegraphics[width=0.120\linewidth, clip]{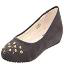}&
\includegraphics[width=0.120\linewidth, clip]{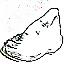}&
\includegraphics[width=0.120\linewidth, clip]{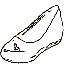} \\
\includegraphics[width=0.120\linewidth, clip]{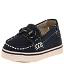}&
\includegraphics[width=0.120\linewidth, clip]{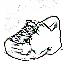}&
\includegraphics[width=0.120\linewidth, clip]{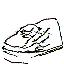} \\
\includegraphics[width=0.120\linewidth, clip]{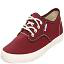}&
\includegraphics[width=0.120\linewidth, clip]{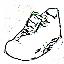}&
\includegraphics[width=0.120\linewidth, clip]{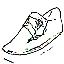} \\
\includegraphics[width=0.120\linewidth, clip]{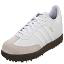}&
\includegraphics[width=0.120\linewidth, clip]{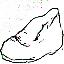}&
\includegraphics[width=0.120\linewidth, clip]{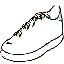} \\
\includegraphics[width=0.120\linewidth, clip]{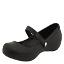}&
\includegraphics[width=0.120\linewidth, clip]{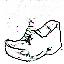}&
\includegraphics[width=0.120\linewidth, clip]{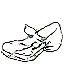} \\
 \multicolumn{3}{c}{(c)}
  \\
\end{tabular}
\end{small}
  \end{minipage}%
\caption{\small Applying unsupervised hyperband for selecting the best configuration. For DiscoGAN and DistanceGAN, we optimize of the number of encoder and decoder layers, batch size and learning rate, while for UNIT, we optmized for the number of encoder and decoder layers, number of resnet layers and learning rate. (a) For each dataset, the first plot is of DiscoGAN and the second is of DistanceGAN. Hyperband optimizes, according to the bound values indicated in blue. The corresponding ground truth errors are shown in orange. Dotted lines represent the best configuration errors, when varying only the number of layers without hyperband (blue for bound and orange for ground truth error). Each graph shows the error of the best configuration selected by hyperband, as a function the number of hyperband iterations. (b) The corresponding hyperparameters of the best configuration as selected by hyperband. (c) Images produced for DiscoGAN's shoes2edges: 1st column is the input, the 2nd is the result of DiscoGAN's default configuration, 3rd is the result of the configuration  selected by our unsupervised Hyperband.} 
  \label{fig:hyperband}
\end{figure*}

\paragraph{Stopping criterion for the non-unique case (Alg.~\ref{algo:when2})}

For testing the stopping criterion suggested in Alg.~\ref{algo:when2}, we plotted the value of the bound and attached a specific sample for a few epochs. For this purpose, we employed DiscoGAN for both $h_1$ and $h_2$, such that the encoder part is shared between them. As we can see in Figs.~\ref{fig:NONUNIQ1}--~\ref{fig:NONUNIQ2}, for smaller values of the bound, we obtain more realistic images and the alignment also improves.

\begin{figure}[t]
    \centering
    \begin{minipage}{0.5\textwidth}
    \centering
        \includegraphics[width=0.8\linewidth, clip]
{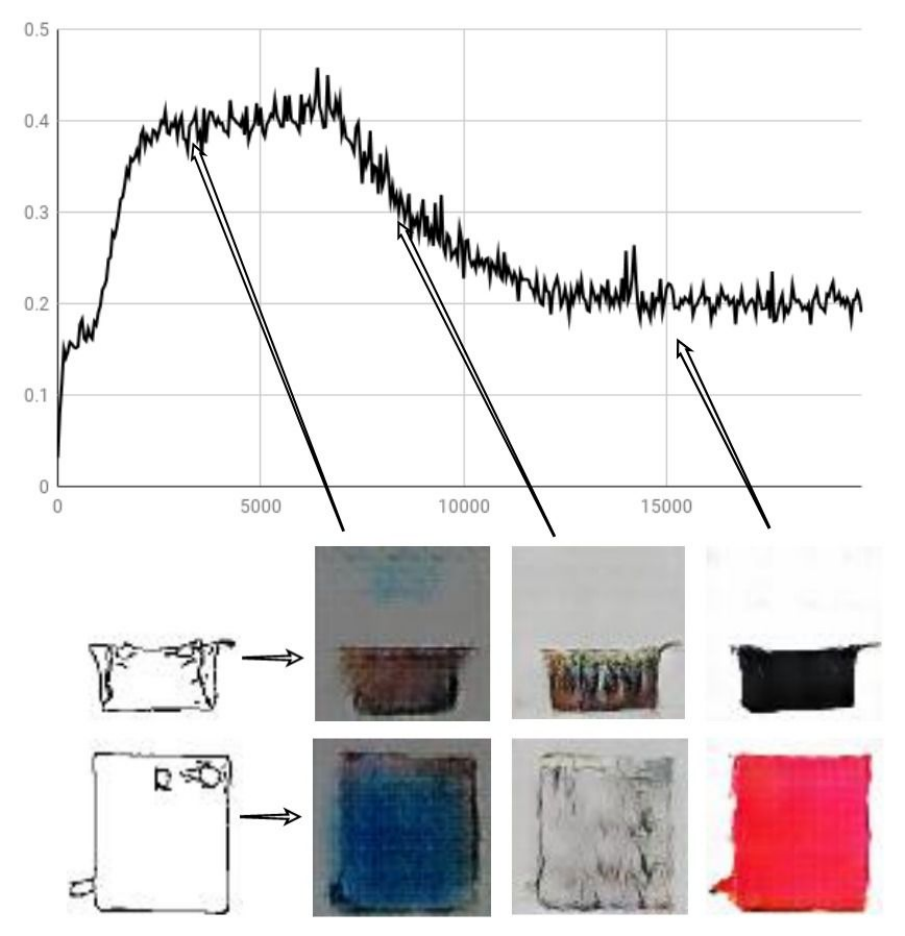}
        \caption{Results for Alg.~\ref{algo:when2} for non-unique translation of Edges to Handbags. The black line is the bound, images are shown for different bound values.}
        \label{fig:NONUNIQ1}
    \end{minipage}%
    \begin{minipage}{0.5\textwidth}
    \centering
        \includegraphics[width=0.8\linewidth, clip]
{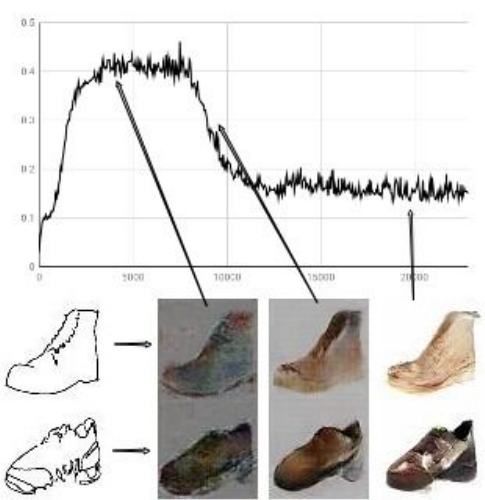}
        \caption{Results for Alg.~\ref{algo:when2} for non-unique translation of Edges to Shoes. The black line is the bound, images are shown for different bound values.}
        \label{fig:NONUNIQ2}
    \end{minipage}
\end{figure}

{\color{black}
\subsection{Estimating the Approximation Error Term}\label{sec:estK}

We conducted an experiment for validating that the term $\mathcal{K} = \inf_{h \in \mathcal{P}_{\omega}} \inf_{d \in \mathcal{C},~\beta(d)\leq 1}\mathcal{K}(h,d;y)$ is small in comparison to $\inf_{h \in \mathcal{P}_{\omega}}\mathbb{E}_{x\sim D_A}[\|h(x)-y(x)\|_2]$. In order to estimate $\mathcal{K}$, we trained a generator $h$ and a discriminator $d$ to minimize $\mathcal{K}(h,d;y) = \mathbb{E}_{x \sim D_A}[\nabla_{y(x)}d(y(x)) - (h(x)-y(x))]$ given supervised data $\{(x_i,y(x_i))\}^{m}_{i=1}$ and similarly we trained a generator $h$ to minimize $\mathbb{E}[\|h(x)-y(y)\|_2]$. The experiment was done using $h$ is the standard $4$-layers architecture generator of DiscoGAN/~DistanceGAN/~CycleGAN and is trained to minimize corresponding term along with its standard losses (i.e., GAN and circularity/distance correlation losses). The discriminator $d$ is trained to minimize $\mathcal{K}(h,d;y)$ along with a constraint to minimize the loss $\frac{1}{m}\sum^{m}_{i=1}\|\Hess_d (y(x_i))\|_2$. We plot the values of the term $\mathcal{K}(h,d;y)$ only for $d$'s that satisfy $\frac{1}{m}\sum^{m}_{i=1}\|\Hess_d (y(x_i))\|_2 \leq 1$. The architecture of $d$ consists of four convolutional layers, each one has $w$ channels, kernel size 4, stride size 2 and a padding value 1. The activation function in each layer is Leaky ReLU with slope 0.2. The number of channels is treated as the width of $d$. 

In order to investigate the effect of the complexity of $d$ on the value of $\mathcal{K}(h,d;y)$, we ran the experiment of discriminators with $w\in \{10,200,500\}$. Fig.~\ref{fig:estK} depicts the results of the comparison of the values of the two terms on test data as a function of the number of iterations. As can be seen, the value of $\mathcal{K}$ is significantly smaller than $\inf_{h \in \mathcal{P}_{\omega}}\mathbb{E}[\|h(x)-y(y)\|_2]$ for all values of $w$, and it also significantly decreases as $w$ increases. This behaviour is consistent over all iterations.

\begin{figure*}[t]
\centering
\hspace{-0.6cm}
  \begin{tabular}{c@{~}c@{~}c}
  \includegraphics[width=0.3\linewidth, clip]{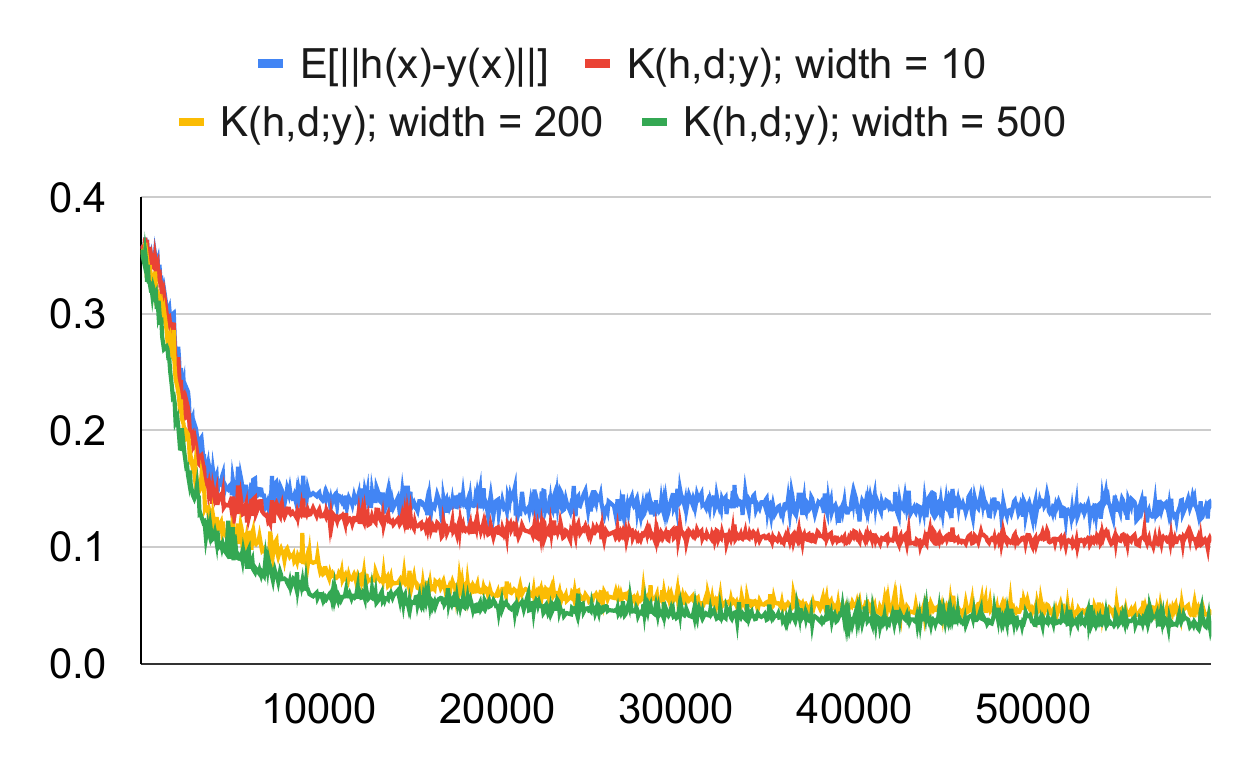} & 
  \includegraphics[width=0.3\linewidth, clip]{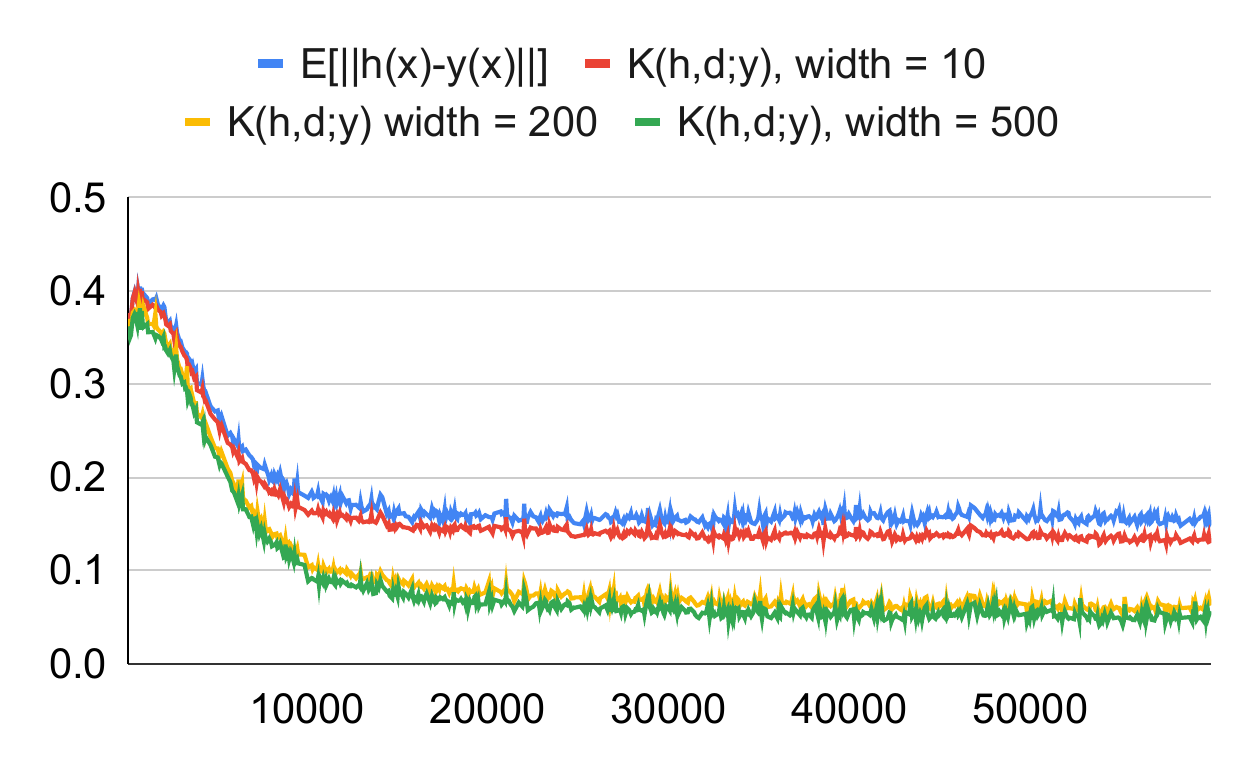} &
  \includegraphics[width=0.3\linewidth, clip]{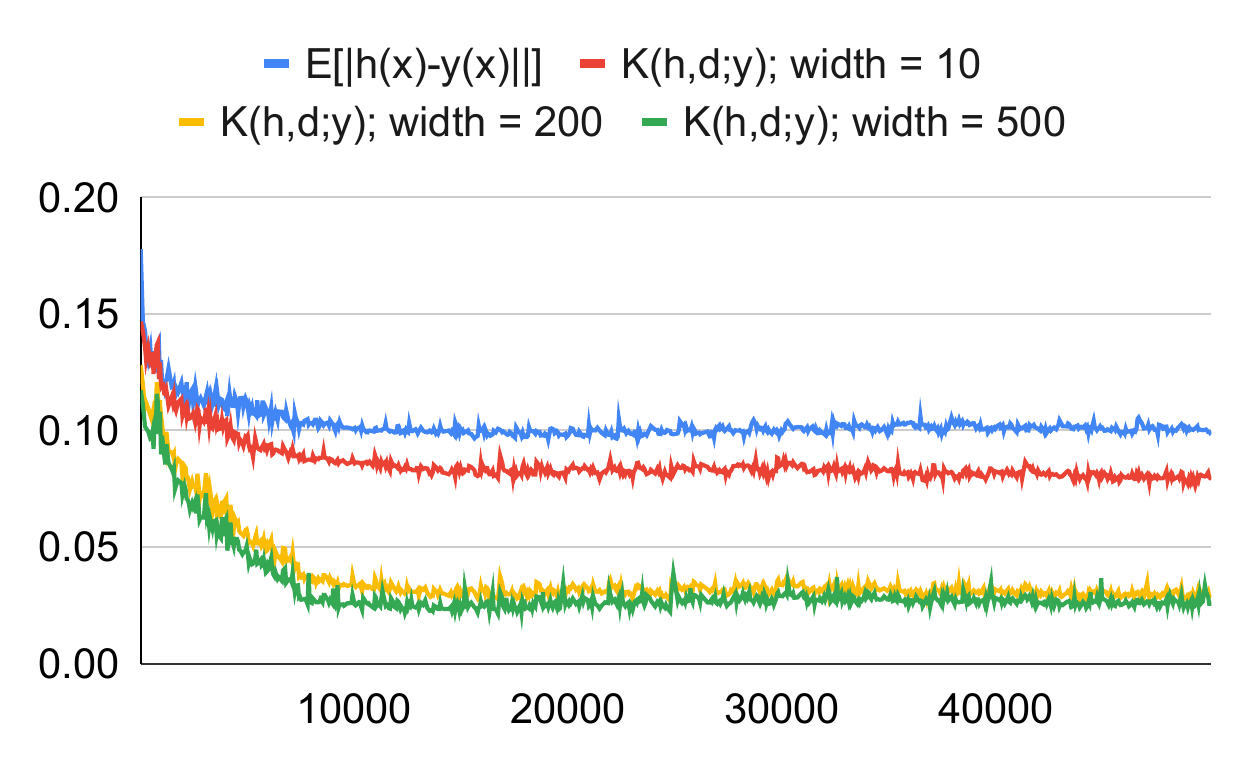} \\
(DistanceGAN, Shoes2Edges) & (DiscoGAN, Bags2Edges) & (CycleGAN, Cityscapes) \\
\end{tabular}
\caption{ {\color{black} {\bf Comparing $\mathcal{K}$ with $\inf_{h}\mathbb{E}[\|h(x)-y(y)\|_2]$.} The x-axis is the training iteration. The blue curve specifies the value of  $\mathbb{E}[\|h(x)-y(y)\|_2]$, the red, yellow and green curves stand for the values of $\mathcal{K}(h,d;y)$ for $d$ of widths $10,200,500$ (resp.).}}
  \label{fig:estK}
\end{figure*}

}

\section{Conclusions}

The recent success in mapping between two domains in an  unsupervised way and without any existing knowledge, other than network hyperparameters, is nothing less than extraordinary and has far reaching consequences. As far as we know, nothing in the existing machine learning or cognitive science literature suggests that this would be possible. 

{\color{black}In Sec.~\ref{sec:dualProof}, we derived a novel risk bound for the unsupervised learning of mappings between domains. The bound takes into account the ability of the hypothesis classes (including both the generator and the discriminator) to model the cross-domain mapping task and the ability to generalize from a finite set of samples. 

This bound leads directly to a method for estimating the success of the learned mapping between the two domains without relying on a validation set. By training pairs of networks that are distant from each other, we are able to obtain a confidence measure on the mapping's outcome. The confidence estimation has application to hyperparameters selection and for performing early stopping.  The bound is extended to the non-unique case mapping case in Sec.~\ref{sec:dualProofExt}.}

\section*{Acknowledgements}

This project has received funding from the European Research Council (ERC) under the European Union's Horizon 2020 research and innovation programme (grant ERC CoG 725974). 

\bibliography{gans}
\clearpage
\appendix

\section{Proofs of the Main Results}\label{sec:proofs}

{\color{black}
\subsection{Useful Lemmas}

\begin{lemma}\label{lem:rhoIneq} Let $\mathcal{C}$ be a symmetric class of functions $d:\mathcal{X} \to \mathbb{R}$ and $D_1, D_2,D_3,D_4$ be four distributions over $\mathcal{X}$, then,
\begin{equation}
\Big\vert \rho_{\mathcal{C}}(D_1, D_2) - \rho_{\mathcal{C}}(D_3, D_4) \Big\vert \leq  \rho_{\mathcal{C}}(D_1,D_3) + \rho_{\mathcal{C}}(D_2,D_4) 
\end{equation}
\end{lemma}

\begin{proof} We consider that:
\begin{equation}
\begin{aligned}
&\Bigg\vert \rho_{\mathcal{C}}(D_1, D_2) - \rho_{\mathcal{C}}(D_3, D_4) \Bigg\vert \\  
=& \Bigg\vert \sup_{d \in \mathcal{C}} \Big\{ \mathbb{E}_{x \sim D_1} [d(x)] - \mathbb{E}_{x \sim D_2} [d(x)] \Big\} - \sup_{d \in \mathcal{C}} \Big\{ \mathbb{E}_{x \sim D_3} [d(x)] - \mathbb{E}_{x \sim D_4} [d(x)] \Big\} \Bigg\vert \\
\leq& \Bigg\vert \sup_{d \in \mathcal{C}} \Big\{ \mathbb{E}_{x \sim D_1} [d(x)] - \mathbb{E}_{x \sim D_2} [d(x)] -  \mathbb{E}_{x \sim D_3} [d(x)] + \mathbb{E}_{x \sim D_4} [d(x)] \Big\} \Bigg\vert \\
\leq& \Bigg\vert \sup_{d \in \mathcal{C}} \Big\{ \mathbb{E}_{x \sim D_1} [d(x)] - \mathbb{E}_{x \sim D_3} [d(x)] \Big\} + \sup_{d \in \mathcal{C}} \Big\{ \mathbb{E}_{x \sim D_2} [d(x)] - \mathbb{E}_{x \sim D_4} [d(x)]\Big\} \Bigg\vert \\
=& \sup_{d \in \mathcal{C}} \Big\{ \mathbb{E}_{x \sim D_1} [d(x)] - \mathbb{E}_{x \sim D_3} [d(x)] \Big\} + \sup_{d \in \mathcal{C}} \Big\{ \mathbb{E}_{x \sim D_2} [d(x)] - \mathbb{E}_{x \sim D_4} [d(x)]\Big\} \\
=& \rho_{\mathcal{C}}(D_1,D_3) + \rho_{\mathcal{C}}(D_2,D_4) \\
\end{aligned}
\end{equation}
The last two equations follow from the definition of $\rho_{\mathcal{C}}$ and the assumption that $\mathcal{C}$ is symmetric.
\end{proof}
}



The following lemma is a variation of the Occam's Razor theorem from~\citep{benaim2017maximally}, where it was used to bound the risk between $h_1\in \mathcal H$ and the target function $y$, when assuming that there is a good approximation for $y$ in, what appears here as $\mathcal{P}$.

\begin{lemma}\label{lem:cvprBound} Let $y\in \mathcal{T}$ be a target function and $\mathcal{P}$ a class of functions. Then, for every function $h_1 \in \mathcal{P}$, we have:
\begin{equation}\label{eq:lemcvpr1}
R_{D_A}[h_1,y] \leq 3\sup\limits_{h_2 \in \mathcal{P}} R_{D_A}[h_1,h_2] + 3\inf\limits_{h \in \mathcal{P}} R_{D_A}[h,y]
\end{equation}
and,
\begin{equation}\label{eq:lemcvpr2}
\sup\limits_{h_1,h_2\in \mathcal{P}} R_{D_A}[h_1,h_2] \leq 6\sup\limits_{h \in \mathcal{P}} R_{D_A}[h,y]
\end{equation}
\end{lemma}

\begin{proof} First, we prove Eq.~\ref{eq:lemcvpr1}. We consider that 
\begin{equation}
\begin{aligned}
\ell(a,c) &= \|a-c\|^2_2 \\
&= \|a-b+b-c\|^2_2 \\
&\leq (\|a-b\|_2+\|b-c\|_2)^2 \\
&= \|a-b\|^2_2 + \|b-c\|^2_2 + 2\|a-b\|_2\cdot \|b-c\|_2 \\
&\leq \|a-b\|^2_2 + \|b-c\|^2_2 + 2\max (\|a-b\|^2_2,\|b-c\|^2_2) \\
&\leq 3(\|a-b\|^2_2+\|b-c\|^2_2) \\
&= 3(\ell(a,b) + \ell(b,c))
\end{aligned}
\end{equation}
Therefore, we have: 
\begin{equation}
\begin{aligned}
R_{D_A}[h_1,y]  
&= \mathbb{E}_{x \sim D_A}[\|h_1(x)-y(x)\|^2_2]\\
&\leq \mathbb{E}_{x \sim D_A}\left[3\|h_1(x)-h^*(x)\|^2_2+3\|h^*(x)-y(x)\|^2_2 \right]\\
&= 3 \left[ R_{D_A}[h_1,h^*] + R_{D_A}[h^*,y] \right] \\
&= 3 \left[ R_{D_A}[h_1,h^*] + \inf\limits_{h \in \mathcal{P}} R_{D_A}[h,y] \right]
\end{aligned}
\end{equation}
where $h^* \in \arg\inf\limits_{h \in \mathcal{P}} R_{D_A}[h^*,y]$. Since $h^* \in \mathcal{P}$, we have: $R_{D_A}[h_1,h^*] \leq \sup\limits_{h_2 \in \mathcal{P}} R_{D_A}[h_1,h_2]$ and the desired inequality follows immediately. By the same argument, we prove Eq.~\ref{eq:lemcvpr2}. We consider that, $R_{D_A}[h_1,h_2] \leq 3\left[R_{D_A}[h_1,y] + R_{D_A}[h_2,y]\right]$. Therefore, $\sup\limits_{h_1,h_2\in \mathcal{P}} R_{D_A}[h_1,h_2] \leq 3\sup\limits_{h_1,h_2\in \mathcal{P}}\left[R_{D_A}[h_1,y] + R_{D_A}[h_2,y]\right] = 6\sup\limits_{h\in \mathcal{P}}R_{D_A}[h,y]$.
\end{proof}

\subsection{Proof of Thm.~\ref{thm:boundGEN}}

The following lemma bounds the generalization risk between a hypothesis $h$ and a target function $y$. The upper bound is a function of $\rho_{\mathcal{C}}(h \circ D_A,D_B) = \sup\limits_{d \in \mathcal{C}} \left\{ \mathbb{E}_{x \sim h \circ D_A} [d(x)] - \mathbb{E}_{x \sim D_B} [d(x)] \right\}$, which is the $\mathcal{C}$-IPM between the distributions $h\circ D_A$ and $D_B$. An additional term expresses the approximation of $h(x)-y(x)$ by the gradient of a function $d \in \mathcal{C}$. Both terms are multiplied by a term that depends on the smoothness of $d$.

\begin{lemma}\label{lem:bound1} Assume the settings of Sec.~\ref{sec:problemformulation} and Sec.~\ref{sec:dualProof}.  Assume that $\mathcal{X}_A \subset \mathbb{R}^N$ and $\mathcal{X}_B \subset \mathbb{R}^M$ are convex and bounded sets. Assume that $\mathcal{C} \subset C^2$. Let $y \in \mathcal{T}$ be target function and $d \in \mathcal{C}$ such that $\beta(d)<2$. Then, for any function $h \in \mathcal{H}$, such that $h:\mathcal{X}_A \to \mathcal{X}_B$, we have:
\begin{equation}
\begin{aligned}
R_{D_A}[h,y] \leq \frac{2\rho_{\mathcal{C}}(h \circ D_A,D_B)}{2-\beta(d)} + \frac{2\sup\limits_{u \in \mathcal{X}_A} \|h(u) - y(u)\|_2 }{2-\beta(d)}\cdot
\mathcal{K}(h,d;y)
\end{aligned}
\end{equation}
\end{lemma}

\begin{proof} First, since each function $f \in \mathcal{H} \cup \mathcal{T}$ is measurable, by a change of variables (cf.~\cite{limits}, Thm.~1.9), we can represent the $\mathcal{C}$-IPM in the following manner:
\begin{align}
\rho_{\mathcal{C}}(h \circ D_A,D_B) &= \sup\limits_{d\in \mathcal{C}} \Big\{ \mathbb{E}_{u\sim h \circ D_A}[d(u)] - \mathbb{E}_{v\sim D_B}[d(v)] \Big\} \nonumber \\
&= \sup\limits_{d\in \mathcal{C}} \Big\{ \mathbb{E}_{u\sim h \circ D_A}[d(u)] - \mathbb{E}_{v\sim y \circ D_A}[d(v)] \Big\} \nonumber \\
&= \sup\limits_{d\in \mathcal{C}} \Big\{ \mathbb{E}_{x\sim D_A}[d \circ h(x)] - \mathbb{E}_{x\sim D_A}[d \circ y (x)] \Big\} \nonumber \\
&= \sup\limits_{d\in \mathcal{C}} \Big\{  \mathbb{E}_{x\sim D_A}[d \circ h(x)- d \circ y (x)] \Big\} \label{eq:repr}
\end{align}
For fixed $d \in \mathcal{C}$ and $z \in \mathcal{X}_B$, we can write the following Taylor expansion (possible since $\mathcal{C} \subset C^2$):
\begin{equation}
\begin{aligned}
d(z+\delta) - d(z) = \left\langle \nabla d(z), \delta \right\rangle  + \frac{1}{2}\left\langle \delta^\top \cdot \Hess_d(u^*),\delta \right\rangle
\end{aligned}
\end{equation}
where $u^*$ is strictly between $z$ and $z+\delta$ (on the line connecting $z$ and $z + \delta$). In particular, for each $d \in \mathcal{C}$ and $x \in \mathcal{X}_A$, if $z = y(x)$ and $\delta = h(x)-y(x)$, we have:
\begin{equation}\label{eq:TaylorRep}
\begin{aligned}
d(h(x)) - d(y(x)) = &\left\langle \nabla_{y(x)} d(y(x)), h(x)-y(x) \right\rangle \\
&+ \frac{1}{2}\left\langle (h(x)-y(x))^\top \cdot \Hess_d(u^*_{d,x}), h(x)-y(x) \right\rangle \\ 
\end{aligned}
\end{equation}
where $u^*_{d,x}$ is strictly between $y(x)$ and $h(x)$ (on the line connecting $y(x)$ and $h(x)$). Therefore, by combining Eqs.~\ref{eq:repr} and~\ref{eq:TaylorRep}, we obtain that for every $d \in \mathcal{C}$, we have:
\begin{equation} 
\begin{aligned}
\rho_{\mathcal{C}}(h \circ D_A,D_B) \geq& \mathbb{E}_{x \sim D_A}[d(h(x)) - d(y(x))]\\
=& \mathbb{E}_{x \sim D_A}\left[\left\langle \nabla_{y(x)} d(y(x)), h(x)-y(x) \right\rangle \right] \\
&+ \frac{1}{2} \mathbb{E}_{x \sim D_A}\left[ \left\langle (h(x)-y(x))^\top \cdot \Hess_d(u^*_{d,x}), (h(x)-y(x)) \right\rangle \right] \\
=&  \mathbb{E}_{x \sim D_A}\left[\|h(x)-y(x)\|^2_2\right] \\
&+ \mathbb{E}_{x \sim D_A}\left[\left\langle \nabla_{y(x)} d(y(x)) - (h(x)-y(x)), h(x)-y(x) \right\rangle \right] \\
&+ \frac{1}{2} \mathbb{E}_{x \sim D_A}\left[ \left\langle (h(x)-y(x))^\top \cdot \Hess_d(u^*_{d,x}),h(x)-y(x) \right\rangle \right]\\
\end{aligned}
\end{equation}
In particular, by $|\mathbb{E}[X]|\leq \mathbb{E}[|X|]$, we have:
\begin{align}
&\rho_{\mathcal{C}}(h \circ D_A,D_B) \nonumber \\
\geq&  \mathbb{E}_{x \sim D_A}\left[\|h(x)-y(x)\|^2_2\right] - \Big\vert \mathbb{E}_{x \sim D_A}\left[\left\langle \nabla_{y(x)} d(y(x)) - (h(x)-y(x)), h(x)-y(x) \right\rangle \right] \Big\vert \nonumber \\
&- \frac{1}{2} \Big\vert \mathbb{E}_{x \sim D_A}\left[ \left\langle (h(x)-y(x))^\top \cdot \Hess_d(u^*_{d,x}),h(x)-y(x) \right\rangle \right] \Big\vert \nonumber \\
\geq&  \mathbb{E}_{x \sim D_A}\left[\|h(x)-y(x)\|^2_2\right] -\mathbb{E}_{x \sim D_A}\left[ \Big\vert \left\langle \nabla_{y(x)} d(y(x)) - (h(x)-y(x)), h(x)-y(x) \right\rangle  \Big\vert\right] \label{eq:W1}\\
&- \frac{1}{2} \mathbb{E}_{x \sim D_A}\left[ \Big\vert \left\langle (h(x)-y(x))^\top \cdot \Hess_d(u^*_{d,x}), h(x)-y(x) \right\rangle \Big\vert\right] \nonumber 
\end{align}
By applying the Cauchy-Schwartz inequality,
\begin{align}
&\big\vert \left\langle \nabla_{y(x)} d(y(x)) - (h(x)-y(x)), h(x)-y(x) \right\rangle  \big\vert \nonumber \\
&\leq \|\nabla_{y(x)} d(y(x)) - (h(x)-y(x))\|_2 \cdot \|h(x)-y(x)\|_2 \nonumber\\
&\leq \|\nabla_{y(x)} d(y(x)) - (h(x)-y(x))\|_2 \cdot \sup\limits_{u \in \mathcal{X}_A}\|h(u)-y(u)\|_2 \label{eq:cs1}
\end{align}
Again, by applying the Cauchy-Schwartz inequality,
\begin{equation}
\begin{aligned}
\Big\vert \left\langle (h(x)-y(x))^\top \cdot \Hess_d(u^*_{d,x}), h(x)-y(x) \right\rangle \Big\vert &\leq \| (h(x)-y(x))^\top \cdot \Hess_d(u^*_{d,x}) \|_2\cdot \|h(x)-y(x)\|_2 \\
&\leq \| \Hess_d(u^*_{d,x})\|_2 \cdot \|h(x)-y(x)\|^2_2 \\
\end{aligned}
\end{equation}
Since $\mathcal{X}_B$ is convex, $y(x),h(x) \in \mathcal{X}_B$ and $u^{*}_{d,x}$ is on the line connecting $y(x)$ and $h(x)$, we have: $u^{*}_{d,x} \in \mathcal{X}_B$. In particular,
\begin{align}
\Big\vert \left\langle (h(x)-y(x))^\top \cdot \Hess_d(u^*_{d,x}), h(x)-y(x) \right\rangle \Big\vert \nonumber &\leq \sup\limits_{z \in \mathcal{X}_B}\| \Hess_d(z)\|_2 \cdot \|h(x)-y(x)\|^2_2 \nonumber\\
&= \beta(d) \cdot \|h(x)-y(x)\|^2_2 \label{eq:cs2}
\end{align}
Therefore, by combining Eqs.~\ref{eq:W1},~\ref{eq:cs1} and~\ref{eq:cs2}, we have: 
\begin{equation}\label{eq:W2}
\begin{aligned}
\rho_{\mathcal{C}}(h \circ D_A,D_B) \geq& \mathbb{E}_{x \sim D_A}\left[\|h(x)-y(x)\|^2_2\right] - \frac{1}{2} \mathbb{E}_{x \sim D_A}\left[ \beta(d) \cdot \|h(x)-y(x)\|_2^2\right] \\
&- \sup\limits_{u \in \mathcal{X}_A} \| h(u)-y(u) \|_2 \cdot  \mathbb{E}_{x \sim D_A}\left[\|\nabla_{y(x)} d(y(x)) - (h(x)-y(x))\|_2 \right] \\  
=& \left( 1 - \frac{\beta(d)}{2}\right) R_{D_A}[h,y]  \\
&- \sup\limits_{u \in \mathcal{X}_A} \| h(u)-y(u) \|_2 \cdot  \mathbb{E}_{x \sim D_A}\left[\|\nabla_{y(x)} d(y(x)) - (h(x)-y(x))\|_2 \right] \\
=& \left( 1 - \frac{\beta(d)}{2}\right) R_{D_A}[h,y]  - \sup\limits_{u \in \mathcal{X}_A} \| h(u)-y(u) \|_2 \cdot \mathcal{K}(h,d;y) 
\end{aligned}
\end{equation}
By combining Eqs.~\ref{eq:W2} and $\beta(d)<2$, we obtain the desired bound.
\end{proof}

The following result is obtained by combining Lem.~\ref{lem:bound1} with Lem.~\ref{lem:cvprBound}. 

\begin{lemma}\label{lem:boundGEN} Assume the setting of Sec.~\ref{sec:problemformulation}. Assume that $\mathcal{X}_A \subset \mathbb{R}^N$ and $\mathcal{X}_B \subset \mathbb{R}^M$ are convex and bounded sets. Assume that $\mathcal{C} \subset C^2$. Let $\mathcal{T}$ be a class target functions and $\mathcal{P}$ a class of candidate functions. Then, for any $y \in \mathcal{T}$, $h \in \mathcal{P}$, such that, $h:\mathcal{X}_A \to \mathcal{X}_B$, $d \in \mathcal{C}$, such that, $\beta(d) < 2$ and function $h_1\in \mathcal{H}$, we have:
\begin{equation}
\begin{aligned}
R_{D_A}[h_1,y] \leq& 3\sup\limits_{h_2 \in \mathcal{P}} R_{D_A}[h_1,h_2] + \frac{6\rho_{\mathcal{C}}(h \circ D_A,D_B)}{2-\beta(d)}  \\
&+ \frac{6\sup\limits_{u \in \mathcal{X}_A} \|h(u) - y(u)\|_2 }{2-\beta(d)} \cdot \mathcal{K}(h,d;y)
\end{aligned}
\end{equation}
\end{lemma}

\begin{proof} Let $y \in \mathcal{T}$, $h \in \mathcal{P}$ such that $h:\mathcal{X}_A \to \mathcal{X}_B$ and $d \in \mathcal{C}$, such that, $\beta(d) < 2$. By Lem.~\ref{lem:bound1}:
\begin{equation}\label{eq:b2}
\begin{aligned}
R_{D_A}[h,y] \leq \frac{2\rho_{\mathcal{C}}(h \circ D_A,D_B)}{2-\beta(d)} + \frac{2\sup\limits_{u \in \mathcal{X}_A} \|h(u) - y(u)\|_2 }{2-\beta(d)} \cdot \mathcal{K}(h,d;y)
\end{aligned}
\end{equation}
In particular, since $h \in \mathcal{P}$, we have: $\inf\limits_{h^* \in \mathcal{P}} R_{D_A}[h^*,y] \leq R_{D_A}[h,y]$. By combining Eq.~\ref{eq:lemcvpr1} (of Lem.~\ref{lem:cvprBound}) with Eq.~\ref{eq:b2}, we obtain the desired inequality.
\end{proof}

\begin{lemma}\label{lem:boundGENgeneral} Assume that $\mathcal{X}_A \subset \mathbb{R}^N$ and $\mathcal{X}_B \subset \mathbb{R}^M$ are convex and bounded sets. Assume that $\mathcal{C} \subset C^2$. Let $\mathcal{T}$ be a class target functions and $\mathcal{P}$ a class of candidate functions. Let $\alpha \in [0,1)$. Then, for any $h_1\in \mathcal{H}$, we have:
\begin{equation}
\begin{aligned}
\inf\limits_{y \in \mathcal{T}} R_{D_A}[h_1,y] \lesssim & \sup\limits_{h_2 \in \mathcal{P}} R_{D_A}[h_1,h_2] + \frac{1}{1-\alpha}\inf_{h,d} \left\{\rho_{\mathcal{C}}(h \circ D_A,D_B)+\inf\limits_{y \in \mathcal{T}}
\mathcal{K}(h,d;y)\right\}
\end{aligned}
\end{equation}
where the infimum is taken over $h\in \mathcal{P}$ (such that, $h:\mathcal{X}_A \to \mathcal{X}_B$) and $d \in \mathcal{C}$, such that, $\beta(d)\leq 1+\alpha$.
\end{lemma}

\begin{proof} Let $h \in \mathcal{P}$, such that, $h:\mathcal{X}_A \to \mathcal{X}_B$, $d \in \mathcal{C}$, such that $\beta(d)\leq 1$ and $y \in \mathcal{T}$. Then, by Lem.~\ref{lem:boundGEN}, for every $h_1 \in \mathcal{P}$, we have:
\begin{equation}
\begin{aligned}
R_{D_A}[h_1,y] \leq& 3\sup\limits_{h_2 \in \mathcal{P}} R_{D_A}[h_1,h_2] + 6\rho_{\mathcal{C}}(h \circ D_A,D_B) \\
&+ 6\sup\limits_{u \in \mathcal{X}_A} \|h(u) - y(u)\|_2 \cdot
\mathcal{K}(h,d;y)  \\
\end{aligned}
\end{equation}
In particular, since $\mathcal{X}_B$ is bounded, there is a constant $L>0$ such that $\sup\limits_{a,b \in \mathcal{X}_B} \|a - b\|_2 \leq L$. Hence, for every $h,y:\mathcal{X}_A \to \mathcal{X}_B$, we have: $\sup\limits_{u \in \mathcal{X}_A} \|h(u) - y(u)\|_2 \leq L$. Therefore, 
\begin{equation}\label{eq:deter}
\begin{aligned}
R_{D_A}[h_1,y] \lesssim & \sup\limits_{h_2 \in \mathcal{P}} R_{D_A}[h_1,h_2] + \frac{1}{1-\alpha}\inf_{h,d} \left\{\rho_{\mathcal{C}}(h \circ D_A,D_B)+\mathcal{K}(h,d;y)\right\}
\end{aligned}
\end{equation}
Finally, by taking $\inf\limits_{y \in \mathcal{T}}$ in both sides of Eq.~\ref{eq:deter}, we obtain the desired inequality. 
\end{proof}



{\color{black}

\begin{lemma}[Cross-Domain Mapping with IPMs]\label{lem:boundGENE} Assume that $\mathcal{X}_A \subset \mathbb{R}^N$ and $\mathcal{X}_B \subset \mathbb{R}^M$ are convex and bounded sets. Assume that $\mathcal{C} \subset C^2$ and $\sup_{d \in \mathcal{C}}\|d\|_{\infty,\mathcal{X}_A \cup \mathcal{X}_B} < \infty$. Then, for any $\delta \in (0,1)$,  with probability at least $1-\delta$ over the selection of $\mathcal{S}_A \sim D^{m_1}_A$ and $\mathcal{S}_B \sim D^{m_2}_B$, for every $\omega \in \Omega$ and  $h_1 \in \mathcal{P}_{\omega}(\mathcal{S}_A,\mathcal{S}_B)$, we have:
\begin{equation}\label{eq:main2}
\begin{aligned}
\inf_{y\in \mathcal{T}} R_{D_A}[h_1,y] \lesssim& \sup\limits_{h_2 \in \mathcal{P}_{\omega}(\mathcal{S}_A,\mathcal{S}_B)} R_{\mathcal{S}_A}[h_1,h_2]\\
&+\frac{1}{1-\alpha}\inf\limits_{h \in \mathcal{P}_{\omega}(\mathcal{S}_A,\mathcal{S}_B)} \left\{ \rho_{\mathcal{C}}(h \circ \mathcal{S}_A,\mathcal{S}_B) +\inf_{y\in \mathcal{T}} \inf_{\substack{d \in \mathcal{C} \\ \beta(d)\leq 1}} 
\mathcal{K}(h,d;y) \right\} \\ 
&+ \hat{\mathscr{R}}_{\mathcal{S}_A}(\mathcal{H}) + \frac{1}{1-\alpha} \left(\hat{\mathscr{R}}_{\mathcal{S}_A}(\mathcal{C}\circ \mathcal{H}) + \hat{\mathscr{R}}_{\mathcal{S}_B}(\mathcal{C})  + \sqrt{\frac{\log(1/\delta)}{\min(m_1,m_2)}} \right)
\end{aligned}
\end{equation}
\end{lemma}

\begin{proof}
By Lem.~\ref{lem:boundGENgeneral}, for any class $\mathcal{P}$, we have:
\begin{equation}
\begin{aligned}
\inf\limits_{y \in \mathcal{T}} R_{D_A}[h_1,y] \lesssim & \sup\limits_{h_2 \in \mathcal{P}} R_{D_A}[h_1,h_2] \\
&+ \frac{1}{1-\alpha} \inf_{h,d} \left\{\rho_{\mathcal{C}}(h \circ D_A,D_B)+\inf\limits_{y \in \mathcal{T}} \mathcal{K}(h,d;y) \right\}
\end{aligned}
\end{equation}
where the infimum is taken over $h\in \mathcal{P}$ and $d \in \mathcal{C}$, such that, $\beta(d)\leq 1$. In particular, for any datasets $\mathcal{S}_A$ and $\mathcal{S}_B$ and $\omega \in \Omega$, we have:
\begin{equation}
\begin{aligned}
\inf_{y\in \mathcal{T}} R_{D_A}[h_1,y] \lesssim &\sup\limits_{h_2 \in \mathcal{P}_{\omega}(\mathcal{S}_A,\mathcal{S}_B)} R_{D_A}[h_1,h_2]\\
&+\frac{1}{1-\alpha} \inf\limits_{h,d} \left\{ \rho_{\mathcal{C}}(h \circ D_A,D_B) + \inf_{y\in \mathcal{T}} \mathcal{K}(h,d;y)\right\}
\end{aligned}
\end{equation}
where the infimum is taken over $h\in \mathcal{P}_{\omega}(\mathcal{S}_A,\mathcal{S}_B)$ and $d \in \mathcal{C}$, such that, $\beta(d)\leq 1+\alpha$. Therefore, we are left to replace the terms $R_{D_A}[h_1,h_2]$ and $\rho_{\mathcal{C}}(h \circ D_A,D_B)$ with their empirical versions, $R_{\mathcal{S}_A}[h_1,h_2]$ and $\rho_{\mathcal{C}}(h \circ \mathcal{S}_A,\mathcal{S}_B)$.

By the Rademacher complexity generalization bound~\citep{koltch} (see also~\citep{Mohri:2012:FML:2371238}, Thm.~3.3), with probability at least $1-\delta/3$ over the selection of $\mathcal{S}_A \sim D^{m_1}_A$, for all $h_1, h_2 \in \mathcal{H}$, we have:
\begin{equation}
\begin{aligned}
R_{D_A}[h_1,h_2] \leq& R_{\mathcal{S}_A}[h_1,h_2] + 2\hat{\mathscr{R}}_{\mathcal{S}_A}(\ell_\mathcal{H}) + 9K^2\sqrt{\frac{\log(6/\delta)}{2m_1}} \\
\end{aligned}
\end{equation}
where $\ell_{\mathcal{H}} := \{\ell(h_1(x),h_2(x)) \mid h_1,h_2 \in \mathcal{H}\}$ and the $9K^2$ term follows from $\ell(h_1(x),h_2(x)) = \|h_1(x)-h_2(x)\|^2_2 \leq 3(\|h_1(x)\|^2_2 +\|h_2(x)\|^2_2) \leq 3K^2$. In addition, $f:u \mapsto \|u\|^2_2$ is a $4K$-Lipschitz continuous function for $u$ of norm bounded by $2K$. 
In particular, with probability at least $1-\delta/3$ over the selection of $\mathcal{S}_A \sim D^{m_1}_A$, for all $h_1 \in \mathcal{H}$ and subset $\mathcal{P} \subset \mathcal{H}$, we have:
\begin{equation}
\sup_{h_2 \in \mathcal{P}} R_{D_A}[h_1,h_2] \lesssim \sup_{h_2 \in \mathcal{P}} R_{\mathcal{S}_A}[h_1,h_2] + \hat{\mathscr{R}}_{\mathcal{S}_A}(\ell_\mathcal{H}) + \sqrt{\frac{\log(1/\delta)}{m_1}}
\end{equation}
Therefore, by selecting $\mathcal{P} := \mathcal{P}_{\omega}(\mathcal{S}_A,\mathcal{S}_B)$, we have:
\begin{equation}
\sup_{h_2 \in \mathcal{P}_{\omega}(\mathcal{S}_A,\mathcal{S}_B)} R_{D_A}[h_1,h_2] \lesssim \sup_{h_2 \in \mathcal{P}_{\omega}(\mathcal{S}_A,\mathcal{S}_B)} R_{\mathcal{S_A}}[h_1,h_2] + \hat{\mathscr{R}}_{\mathcal{S}_A}(\ell_\mathcal{H}) + \sqrt{\frac{\log(1/\delta)}{m_1}}
\end{equation}
Next, we would like to replace the $\mathcal{C}$-IPM with its empirical counterpart. By Lem.~\ref{lem:rhoIneq}, we have:
\begin{equation}
\begin{aligned}
\Big\vert \rho_{\mathcal{C}}(h \circ D_A, D_B) - \rho_{\mathcal{C}}(h \circ \mathcal{S}_A, \mathcal{S}_B) \Big\vert \leq \rho_{\mathcal{C}}(h \circ D_A, h \circ \mathcal{S}_A) + \rho_{\mathcal{C}}(D_B, \mathcal{S}_B)
\end{aligned}
\end{equation}
Again, by the Rademacher complexity generalization bound, with probability at least $1-\delta/3$ over the selection of $\mathcal{S}_B \sim D^{m_2}_B$, for all $d \in \mathcal{C}$, we have:
\begin{equation}
\begin{aligned}
\mathbb{E}_{x \sim D_B} [d(x)] - \frac{1}{m_2} \sum_{x \in \mathcal{S}_B} d(x) &\leq 2\hat{\mathscr{R}}_{\mathcal{S}_B}(\mathcal{C}) + 3 \sup_{d \in \mathcal{C}} \|d\|_{\infty, \mathcal{X}_A \cup \mathcal{X}_B}\sqrt{\frac{\log(6/\delta)}{2m_2}}\\
&\lesssim \hat{\mathscr{R}}_{\mathcal{S}_B}(\mathcal{C}) + \sqrt{\frac{\log(1/\delta)}{m_2}}\\
\end{aligned}
\end{equation}
In particular,
\begin{equation}
\begin{aligned}
\rho_{\mathcal{C}}(D_B,\mathcal{S}_B) 
&= \sup_{d \in \mathcal{C}} \left\{ \mathbb{E}_{x \sim D_B} [d(x)] - \frac{1}{m_2} \sum_{x \in \mathcal{S}_B} d(x) \right\} \\
&\lesssim \hat{\mathscr{R}}_{\mathcal{S}_B}(\mathcal{C}) + \sqrt{\frac{\log(1/\delta)}{m_2}}
\end{aligned}
\end{equation}
Similarly, with probability at least $1-\delta/3$ over the selection of $\mathcal{S}_A \sim D^{m_1}_A$, for all $d  \in \mathcal{C}$ and $h \in \mathcal{H}$, we have the desired:
\begin{equation}
\begin{aligned}
\rho_{\mathcal{C}}(h\circ D_A,h \circ \mathcal{S}_A) 
&= \sup_{d \in \mathcal{C}} \left\{\mathbb{E}_{x \sim D_A} [d(h(x))] - \frac{1}{m_2} \sum_{x \in \mathcal{S}_A} d(h(x)) \right\} \\
&\lesssim \hat{\mathscr{R}}_{\mathcal{S}_A}(\mathcal{C} \circ \mathcal{H}) + \sqrt{\frac{\log(1/\delta)}{m_1}}
\end{aligned}
\end{equation}
Therefore, by the union bound, with probability at least $1-2\delta/3$ over the selection of both $\mathcal{S}_A$ and $\mathcal{S}_B$, for every $h \in \mathcal{H}$, we have:
\begin{equation}
\begin{aligned}
\rho_{\mathcal{C}}(h \circ D_A, D_B) \leq& \rho_{\mathcal{C}}(h \circ \mathcal{S}_A, \mathcal{S}_B) + \hat{\mathscr{R}}_{\mathcal{S}_B}(\mathcal{C}) + \hat{\mathscr{R}}_{\mathcal{S}_A}(\mathcal{C} \circ \mathcal{H}) \\
&+ \sqrt{\frac{\log(1/\delta)}{m_1}} + \sqrt{\frac{\log(1/\delta)}{m_2}}
\end{aligned}
\end{equation}
Finally, by the union bound, with probability at least $1-\delta$ over the selection of both $\mathcal{S}_A$ and $\mathcal{S}_B$, for every $h_1,h_2 \in \mathcal{H}$ and $h \in \mathcal{P}_{\omega}(\mathcal{S}_A,\mathcal{S}_B)$, we have:
\begin{equation}
\begin{aligned}
\inf_{y\in \mathcal{T}} R_{D_A}[h_1,y] \lesssim& \sup\limits_{h_2 \in \mathcal{P}_{\omega}(\mathcal{S}_A,\mathcal{S}_B)} R_{D_A}[h_1,h_2]+\frac{1}{1-\alpha} \inf\limits_{h,d} \left\{ \rho_{\mathcal{C}}(h \circ D_A,D_B) + \inf_{y\in \mathcal{T}} \mathcal{K}(h,d;y)  \right\} \\
\lesssim& \sup\limits_{h_2 \in \mathcal{P}_{\omega}(\mathcal{S}_A,\mathcal{S}_B)} R_{\mathcal{S}_A}[h_1,h_2] + \frac{1}{1-\alpha} \inf\limits_{h,d} \left\{ \rho_{\mathcal{C}}(h \circ \mathcal{S}_A,\mathcal{S}_B) 
+ \inf_{y\in \mathcal{T}} \mathcal{K}(h,d;y)  \right\} \\
& + \hat{\mathscr{R}}_{\mathcal{S}_A}(\mathcal{H}) + \frac{1}{1-\alpha} \left(\hat{\mathscr{R}}_{\mathcal{S}_A}(\mathcal{C} \circ \mathcal{H}) + \hat{\mathscr{R}}_{\mathcal{S}_B}(\mathcal{C}) + \sqrt{\frac{\log(1/\delta)}{\min(m_1,m_2)}}\right) \\
\end{aligned}
\end{equation}
\end{proof}
}


{\color{black}Thm.~\ref{thm:boundGEN} follow immediately from the above lemma by taking $\alpha=0$.}

\newpage

\begin{table}[H]
\centering
\caption{{\color{black}Comparing the averaged VGG input-output descriptor similarity $\mathbb{E}_{x \sim D_A}[cs(f(x),f(h(x)))]$ and discrepancy $\rho_{\mathcal{C}}(h \circ D_A,D_B)$ of a generator $h$, when varying its number of layers $k$. ``w/w.o circ'' are short-hands that specify whether the generator was trained with or without circularity losses. The averages are taken with respect to the test data.} }
\label{tab:var_layers_value}
\begin{tabular}{lcccccccc}
                             & & $k=4$     & $k=6$     & $k=8$     & $k=10$     & $k=12$     & $k=14$     \\
\midrule
Male to Female & Discrepancy (w circ) & 0.521 & 0.203  & 0.091 & 0.094 & 0.080 & 0.084 \\ 
               &  VGG similarity (w circ) & 0.301 & 0.269 &  0.103  & 0.106 & 0.096 & 0.110   \\
               & Discrepancy (w.o circ) & 0.501 & 0.213 & 0.102 & 0.091 & 0.079 & 0.82 \\ 
               &  VGG similarity (w.o circ) & 0.332 & 0.292 & 0.110 & 0.115 & 0.132 & 0.117 \\ 
\midrule
Female to Male & Discrepancy (w circ) & 0.872 & 0.122  & 0.155  & 0.075 & 0.074 & 0.091 \\                  
               &  VGG similarity (w circ) & 0.313  & 0.287  & 0.118  & 0.109 & 0.095 & 0.104  \\
               & Discrepancy (w.o circ) & 0.807 & 0.132  & 0.163  & 0.095 & 0.072 & 0.102 \\   
               &  VGG similarity (w.o circ) & 0.298 & 0.283  & 0.117  & 0.115 & 0.090 & 0.094  \\
\midrule
Blond to Black Hair & Discrepancy  (w circ) & 0.447 & 0.204 & 0.092 & 0.082 & 0.084 & 0.081 \\
                    & VGG similarity (w circ) & 0.395 & 0.293 & 0.260  & 0.136 & 0.101 &   0.097    \\
                    & Discrepancy (w.o circ) & 0.431 & 0.212 & 0.087 & 0.092 & 0.098 & 0.078 \\
                    &  VGG similarity (w.o circ) & 0.415 & 0.313 & 0.254  & 0.113 & 0.121 &   0.109    \\
\midrule
Black to Blond Hair & Discrepancy (w circ) & 0.663 & 0.264 & 0.071 & 0.068 & 0.074 & 0.082 \\
                    & VGG similarity (w circ)  & 0.347 & 0.285  & 0.245  & 0.113 & 0.093 &   0.097    \\
                    & Discrepancy (w.o circ) & 0.693 & 0.271 & 0.062 & 0.081 & 0.097 & 0.059 \\
                    &  VGG similarity (w.o circ) & 0.361 & 0.273  & 0.258  & 0.121 & 0.071 &   0.078    \\
\midrule
Eyeglasses & Discrepancy (w circ) & 0.311 & 0.144 & 0.065 & 0.062 & 0.058 & 0.051 \\
to Non-Eyeglasses	& VGG similarity (w circ) & 0.493  & 0.402 & 0.377 & 0.173 & 0.153 & 0.148 \\
                    & Discrepancy (w.o circ) & 0.303 & 0.122 & 0.061 & 0.052 & 0.054 & 0.067 \\
                    &  VGG similarity (w.o circ) & 0.531  & 0.433 & 0.353 & 0.151 & 0.122 & 0.141 \\
\midrule
Non Eyeglasses  & Discrepancy (w circ) & 0.542 & 0.528 & 0.226 & 0.243 & 0.097 & 0.085 \\
to Eyeglasses   & VGG similarity (w circ) & 0.481 & 0.382 & 0.377 & 0.131 & 0.138 & 0.137 \\
                & Discrepancy (w.o circ) & 0.512 & 0.502 & 0.193 & 0.186 & 0.084 & 0.065 \\
               &  VGG similarity (w.o circ) & 0.499 & 0.363 & 0.341 & 0.195 & 0.171 & 0.146 \\
\bottomrule      
\end{tabular}
\end{table}

\begin{table}[ht]\caption{Norms of the various mappings $h$ for mapping Males to Females using the DiscoGAN architecture. 
(b) Norms of $18$-layer networks that approximates the mappings obtained with a varying number of layers.
}\label{tab:a001} 

\begin{center}
\begin{tabular}{llccccc}
\toprule
& & \multicolumn{5}{c}{----------- Number of layers ------------}\\
& Norm & 4 & 6 & 8 & 10 & 12 \\
\midrule
A to B &L1 norm & 6382  & 23530 & 36920  & 44670 & 71930 \\
&Average L1 norm per layer & 1064  & 2353  & 2637   & 2482   & 3270 \\
&L2 norm & 18.25 & 29.24 & 28.44  & 31.72  & 36.57 \\
&Average L2 norm per layer & 7.084 & 8.353 & 7.154  & 6.708  & 7.009 \\
\midrule
B to A&L1 norm & 6311  & 21240 & 31090  & 37380 & 64500 \\
&Average L1 norm per layer & 1052  & 2124  & 2221   & 2077   & 2932 \\
&L2 norm & 18.36 & 26.79 & 25.85  & 28.36  & 34.99 \\
&Average L2 norm per layer & 7.161 & 7.757 & 6.552  & 6.058  & 6.771 \\
\bottomrule
\end{tabular}\\~\\
(a)
\end{center}
\begin{center}
\begin{tabular}{llccccc}
\toprule
& & \multicolumn{5}{c}{----------- Number of layers ------------}\\
& Norm & 4 & 6 & 8 & 10 & 12 \\
\midrule
A to B & L1 norm & 317200 & 228700 & 356500 & 247200 & 164200  \\
&Average L1 norm per layer & 9329 & 6726 & 10485 & 7271 & 4829  \\ 
&L2 norm  &528.1 & 401.7 & 559.6 & 410.1 & 346.8 \\ 
&Average L2 norm per layer &3.031 & 2.284 & 3.242 &2.257 & 1.890 \\
\midrule
B to A&L1 norm & 316900 & 194500 & 353900 & 171500 & 228900 \\
&Average L1 norm per layer & 9323 & 5719 & 10410 & 5045 & 6733\\
&L2 norm & 523.2 & 375.9 & 555.7 & 346.5 & 373.3  \\
&Average L2 norm per layer & 3.003 & 2.029 &  3.210 & 1.921 & 2.289  \\
\bottomrule
\end{tabular}\\~\\
(b) \\
\end{center}

\end{table}

\begin{figure}[t]
  \centering
 \includegraphics[width=0.6\linewidth,clip]{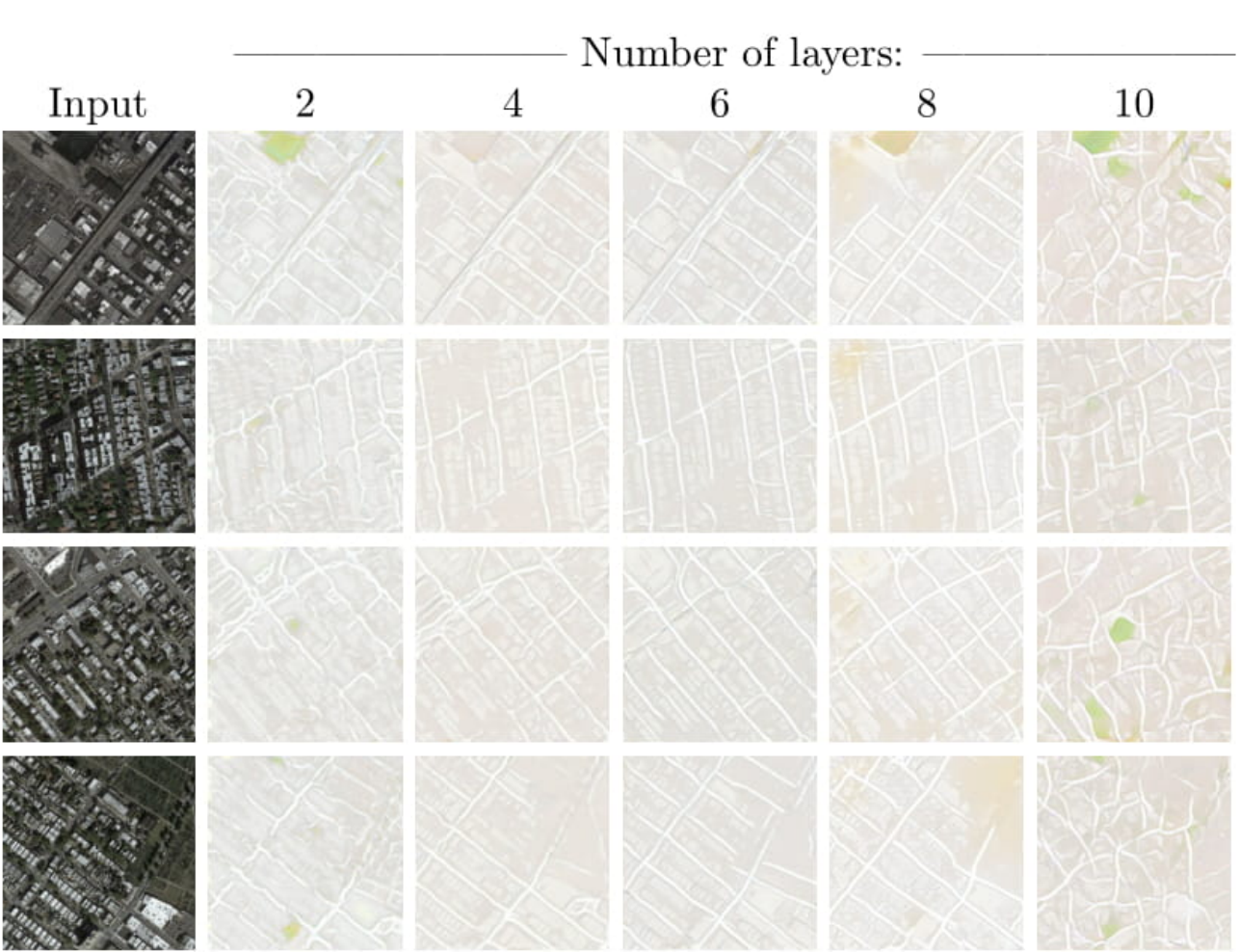}
  \caption{\label{fig:sw001}  {\bf Results of varying the number of layers of the generator.} Results of CycleGAN on Aerial View Images to Maps transfer. The best results are obtained for $4$ or $6$ layers. For more than $6$ layers, the alignment is lost.} 
\end{figure}

\begin{figure}[t]
  \centering
 \includegraphics[width=0.6\linewidth,clip]{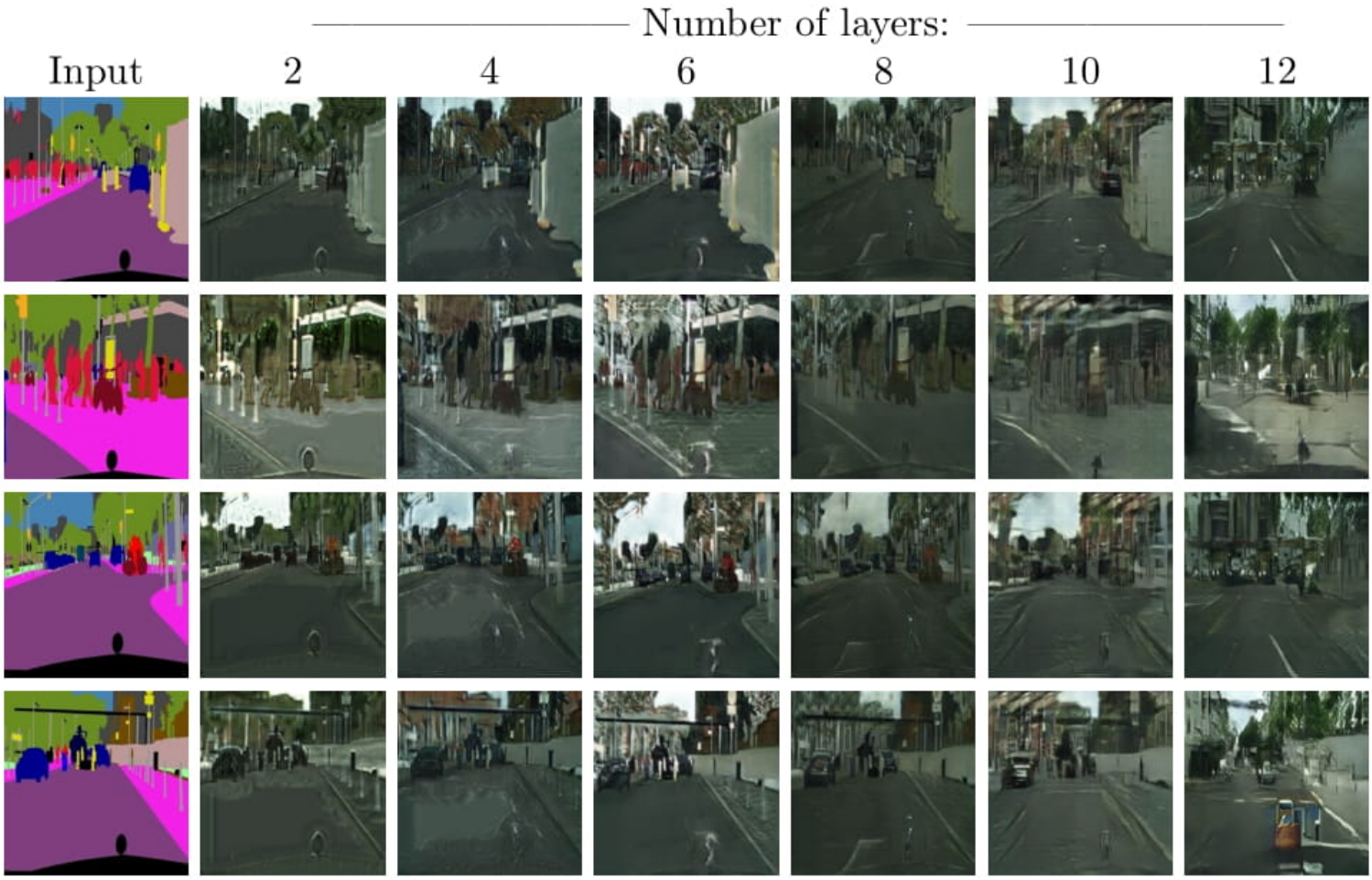}
  \caption{\label{fig:sw002} {\bf Results of varying the number of layers of the generator:} CycleGAN on Segmentation to Image transfer. }
\end{figure}

\begin{figure}[t]
  \centering
 \includegraphics[width=0.6\linewidth,clip]{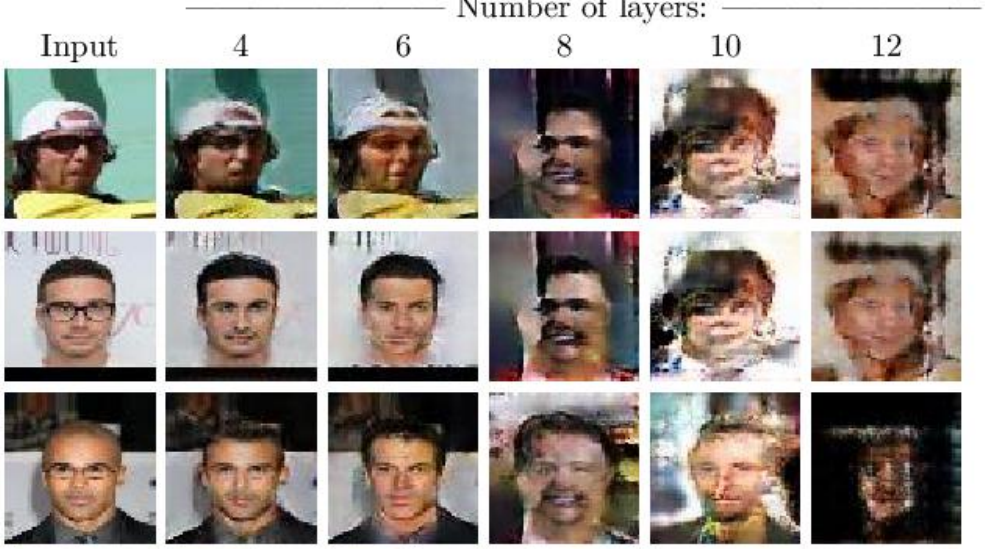}
  \caption{\label{fig:sw003} {\bf Results of varying the number of layers of the generator:} WGAN on eyeglasses removal. }
\end{figure}

\end{document}